\documentclass[twocolumn,twoside]{IEEEtran}

\usepackage{etex}
\reserveinserts{28}
\usepackage{setspace,cite}

\usepackage[dvips]{graphicx}
\usepackage{multirow,multicol}
\usepackage[table]{xcolor}
\usepackage{ctable}

\usepackage[tbtags]{amsmath}
\usepackage{amsbsy}
\usepackage{amssymb}
\usepackage{amsfonts}

\usepackage{graphicx}
\usepackage{pstricks}
\usepackage{pst-node}
\usepackage{pstricks-add}
\usepackage{url}
\usepackage{algorithmic}
\usepackage{algorithm}
\usepackage{balance}
\usepackage{rotating}
\usepackage{multirow}
\usepackage{subcaption}
\usepackage{fancyvrb}
\usepackage{latexsym}
\usepackage{verbatim}

\newtheorem{proposition}{Proposition}
\newtheorem{lemma}{Lemma}
\newtheorem{corollary}{Corollary}
\newtheorem{remark}{Remark}

\newtheorem{example}{Example}

\newtheorem{assumption}{Assumption}

{\begin{list}               
    {$\bullet$ \hfill}{
        \setlength{\leftmargin}{\parindent}
        \setlength{\parsep}{0.04\baselineskip}
        \setlength{\itemsep}{0.5\parsep}
        \setlength{\labelwidth}{\leftmargin}
        \setlength{\labelsep}{0em}}
    }
{\end{list}}

\providecommand{\eref}[1]{\eqref{#1}}  
\providecommand{\cref}[1]{Chapter~\ref{#1}}

\providecommand{\fref}[1]{Figure~\ref{#1}}

\providecommand{\R}{\ensuremath{\mathbb{R}}}

\providecommand{\I}{\ensuremath{\mathbb{I}}}
\providecommand{\Pb}{\ensuremath{\mathbb{P}}}

\providecommand{\bydef}{\overset{\text{def}}{=}}

\renewcommand{\vec}[1]{\ensuremath{\boldsymbol{#1}}}
\providecommand{\mat}[1]{\ensuremath{\boldsymbol{#1}}}

\providecommand{\calN}{\mathcal{N}}

\providecommand{\calP}{\mathcal{P}}

\providecommand{\mD}{\mat{D}}

\providecommand{\mI}{\mat{I}}

\providecommand{\mP}{\mat{P}}

\providecommand{\mW}{\mat{W}}

\providecommand{\ve}{\vec{e}}

\providecommand{\vp}{\vec{p}}

\providecommand{\vu}{\vec{u}}

\providecommand{\vw}{\vec{w}}
\providecommand{\vx}{\vec{x}}
\providecommand{\vy}{\vec{y}}
\providecommand{\vz}{\vec{z}}

\providecommand{\mSigma}{\mat{\Sigma}}
\providecommand{\mTheta}{\mat{\Theta}}

\providecommand{\mSigmaNLM}{\mat{\Sigma}_{\mathrm{NLM}}}
\providecommand{\mSigmahat}{\mat{\widehat{\Sigma}}}

\providecommand{\vmu}{\vec{\mu}}

\providecommand{\zhat}{\widehat{z}}
\providecommand{\sigmahat}{\widehat{\sigma}}

\providecommand{\vzhat}{\boldsymbol{\widehat{z}}}

\providecommand{\vone}{\vec{1}}

\newcommand{\argmin}[1]{\mathop{\underset{#1}{\mbox{argmin}}}}

\newcommand{\diag}[1]{\mathop{\mathrm{diag}\left\{#1\right\}}}

\newcommand{\trace}[1]{\mathop{\mathrm{Tr}\left\{#1\right\}}}

\newgray{Gray08}{0.8}
\newgray{Gray07}{0.7}
\newgray{Gray06}{0.6}
\newgray{Gray05}{0.5}

\title{Understanding Symmetric Smoothing Filters:\\ A Gaussian Mixture Model Perspective}
\author{Stanley~H.~Chan,~\IEEEmembership{Member,~IEEE}, Todd~Zickler,~\IEEEmembership{Member,~IEEE}, and Yue~M.~Lu,~\IEEEmembership{Senior Member,~IEEE}
\thanks{S.~H.~Chan is with the School of Electrical and Computer Engineering, and the Department of Statistics, Purdue University, West Lafayette, IN 47907, USA. Email: \texttt{stanleychan@purdue.edu}.}
\thanks{T.~Zickler and Y.~M.~Lu are with the John A. Paulson School of Engineering and Applied Sciences, Harvard University, Cambridge, MA 02138, USA. E-mails: \texttt{\{zickler,yuelu\}@seas.harvard.edu}.}
\thanks{S.~H.~Chan completed part of this work at Harvard University in 2012-2014. This work was supported in part by the Croucher Foundation Postdoctoral Research Fellowship, and in part by the U.S. National Science Foundation under Grant CCF-1319140. Preliminary material in this paper was presented at the IEEE International Conference on Image Processing (ICIP), Quebec City, Sep 2015.}}

\graphicspath{{pix/}}

\begin{document}
\maketitle

\begin{abstract}
Many patch-based image denoising algorithms can be formulated as applying a smoothing filter to the noisy image. Expressed as matrices, the smoothing filters must be row normalized so that each row sums to unity. Surprisingly, if we apply a column normalization before the row normalization, the performance of the smoothing filter can often be significantly improved. Prior works showed that such performance gain is related to the Sinkhorn-Knopp balancing algorithm, an iterative procedure that symmetrizes a row-stochastic matrix to a doubly-stochastic matrix. However, a complete understanding of the performance gain phenomenon is still lacking.

In this paper, we study the performance gain phenomenon from a statistical learning perspective. We show that Sinkhorn-Knopp is equivalent to an Expectation-Maximization (EM) algorithm of learning a Gaussian mixture model of the image patches. By establishing the correspondence between the steps of Sinkhorn-Knopp and the EM algorithm, we provide a geometrical interpretation of the symmetrization process. This observation allows us to develop a new denoising algorithm called Gaussian mixture model symmetric smoothing filter (GSF). GSF is an extension of the Sinkhorn-Knopp and is a generalization of the original smoothing filters. Despite its simple formulation, GSF outperforms many existing smoothing filters and has a similar performance compared to several state-of-the-art denoising algorithms.
\end{abstract}


\section{Introduction}
\label{sec:introduction}

Smoothing filters are a class of linear and nonlinear operators that gains significant attentions in image denoising recently. The formulations of these operators are simple: Consider a noisy observation $\vy \in \R^n$ of a clean image $\vz \in \R^n$ corrupted by additive i.i.d. Gaussian noise. A smoothing filter is a matrix $\mW \in \R^{n \times n}$ that generates a denoised estimate $\vzhat$ as
\begin{equation}
\vzhat = \mD^{-1}\mW \vy,
\label{eq:standard nlm}
\end{equation}
where $\mD \bydef \diag{\mW\vone}$ is a diagonal matrix for normalization so that each row of $\mD^{-1}\mW$ sums to unity. The formulation in \eref{eq:standard nlm} is very general, and many denoising algorithms are smoothing filters, e.g., Gaussian filter~\cite{Wand_Jones_1995}, bilateral filter~\cite{Paris_Durand_2009}, non-local means (NLM)~\cite{Buades_Coll_2005_Journal}, locally adaptive regression kernel (LARK)~\cite{Takeda_Farsiu_Milanfar_2007}, etc. Note that some of these filters are linear (e.g., Gaussian filter) whereas some are nonlinear (e.g. non-local means). There are interesting graph-theoretic interpretations of the smoothing filters \cite{Milanfar_2013a,Talebi_Zhu_Milanfar_2013,Meyer_Shen_2012,Peyre_Bougleux_Cohen_2008,Peyre_2009,Lezoray_Grady_2012}, and there are also fast algorithms to compute the smoothing filters \cite{Chan_Zickler_Lu_2014, Talebi_Milanfar_2014, Mahmoudi_Sapiro_2005, Adams_Gelfand_2009, Gastal_Oliveira_2012, Bhujle_Chaudhuri_2014}.

\subsection{Motivation: A Surprising Phenomenon}
While smoothing filters work well for many denoising problems, it was observed in \cite{Milanfar_2013b,Chan_Zickler_Lu_2013,Chan_Zickler_Lu_2015} that their performance can be further improved by modifying \eref{eq:standard nlm} as
\begin{equation}
\vzhat = \mD_r^{-1} \mW \mD_c^{-1} \vy,
\label{eq:sinkhorn one step}
\end{equation}
where $\mD_c \bydef \diag{\mW^T\vone}$ is a diagonal matrix that normalizes the \emph{columns} of $\mW$, and $\mD_r \bydef \diag{ \mW \mD_c^{-1}\vone }$ is a diagonal matrix that normalizes the \emph{rows} of $\mW \mD_c^{-1}$. In other words, we modify \eref{eq:standard nlm} by introducing a column normalization step before applying the row normalization.

Before discussing the technical properties of \eref{eq:standard nlm} and \eref{eq:sinkhorn one step}, we first provide some numerical results to demonstrate an interesting phenomenon. In \fref{fig:psnr_gain}, we crop the center $100\times 100$ region of 10 standard clean images. We generate noisy observations by adding i.i.d. Gaussian noise of standard deviation $\sigma=20/255$ to each clean image. These noisy images are then denoised by \eref{eq:standard nlm} and \eref{eq:sinkhorn one step}, respectively. The weight matrix $\mW$ is chosen as the one defined in the non-local means (NLM) \cite{Buades_Coll_2005_Journal}. To ensure fair comparison, we choose the best parameter $h_r$, the range parameter in NLM, for both methods. The patch size is set as $5 \times 5$ and the neighborhood search window is set as $21 \times 21$. The experiment is repeated for 20 independent Monte-Carlo trials to average out the randomness caused by different realizations of the i.i.d. Gaussian noise.

The results of this experiment are shown at the bottom of \fref{fig:psnr_gain}. It is perhaps a surprise to see that \eref{eq:sinkhorn one step}, which is a simple modification of \eref{eq:standard nlm}, improves the PSNR by more than 0.23 dB on average. Another puzzling observation is that if we repeatedly apply the column-row normalization, the PSNR does not always increase as more iterations are used. \fref{fig:sinkhorn_experiment} presents the result. In this experiment, we fix the NLM parameter $h_r$ to its optimal value when using the column-row normalization, i.e., \eref{eq:sinkhorn one step}. For 5 out of the 10 images we tested, the PSNR values actually drop after the first column-row normalization.

\begin{figure*}[t]
\centering
\begin{tabular}{c}
\includegraphics[width=1\linewidth]{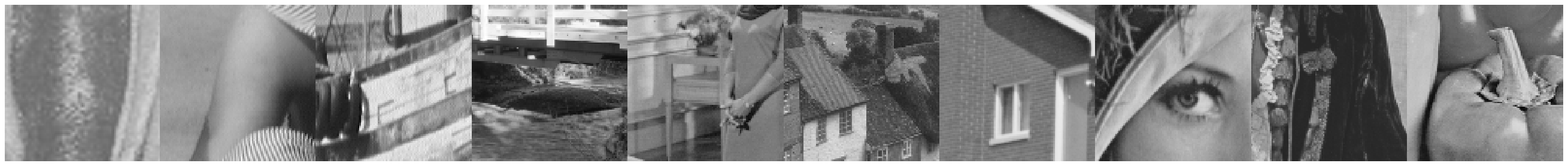}\\
\begin{tabular}{ccccccccccc}
\hline
                             &   \multicolumn{10}{c}{Image No.} \\
Smoothing Filter             &   1       &   2       &   3       &   4       &   5       &   6       &   7       &   8       &   9       &   10\\
\hline\hline
$\mD^{-1}\mW$	             & 30.54 	& 31.42 	& 27.41 	& 26.81 	& 27.89 	& 27.01 	& 30.54 	& 29.16 	& 28.41 	& 29.02 \\
$h_r$                        & 0.64$\sigma$ 	& 0.69$\sigma$ 	& 0.69$\sigma$ 	& 0.71$\sigma$ 	& 0.64$\sigma$ 	
                             & 0.64$\sigma$ 	& 0.71$\sigma$ 	& 0.69$\sigma$ 	& 0.71$\sigma$ 	& 0.71$\sigma$\\
\hline
$\mD_r^{-1}\mW\mD_c^{-1}$	 & 30.64 	& 31.76 	& 27.64 	& 26.98 	& 28.18 	& 27.08 	& 30.87 	& 29.44 	& 28.47 	& 29.39 \\
$h_r$                        & 0.71$\sigma$ 	& 0.77$\sigma$ 	& 0.77$\sigma$ 	& 0.77$\sigma$ 	& 0.71$\sigma$ 	
                             & 0.71$\sigma$ 	& 0.79$\sigma$ 	& 0.77$\sigma$ 	& 0.77$\sigma$ 	& 0.79$\sigma$\\
\hline
PSNR Improvement             & +0.10 	& +0.34 	& +0.23 	& +0.17 	& +0.29 	& +0.07 	& +0.33 	& +0.28 	& +0.05 	& +0.37 \\
\hline
\end{tabular}
\end{tabular}
\caption{[Top] $100 \times 100$ testing images. Each image is corrupted by i.i.d Gaussian noise of $\sigma = 20/255$. [Bottom] PSNR values of the denoised image using $\mD^{-1}\mW$ and $\mD_r^{-1}\mW\mD_c^{-1}$, and the respective optimal NLM parameter $h_r$. All PSNR values in this table are averaged over 20 independent Monte-Carlo trials.}
\label{fig:psnr_gain}
\end{figure*}

The above experiment piqued our curiosity and led us to a basic question: Why would the column-row normalization improve the denoising performance? Insights gained from a better understanding of the underlying mechanism could potentially lead to a more systematic procedure that can generalize the operations in \eref{eq:sinkhorn one step} and further improve the denoising performance. The goal of this paper is to address this issue and propose a new algorithm.

\subsection{Sinkhorn-Knopp Balancing Algorithm}
To the best of our knowledge, the above performance gain phenomenon was first discussed by Milanfar in \cite{Milanfar_2013b}, where it was shown that if we repeatedly apply the column-row normalization we would obtain an iterative procedure called the Sinkhorn-Knopp balancing algorithm \cite{Sinkhorn_1964,Sinkhorn_Knopp_1967} (or Sinkhorn-Knopp, for short.) As illustrated in Algorithm~\ref{alg:sinkhorn}, Sinkhorn-Knopp is a simple algorithm that repeatedly applies the column and row normalization until the smoothing filter converges. For example, the smoothing filter defined by \eref{eq:sinkhorn one step} is the result of applying Sinkhorn-Knopp for one iteration.

\begin{figure}[t]
\centering
\includegraphics[width=0.95\linewidth]{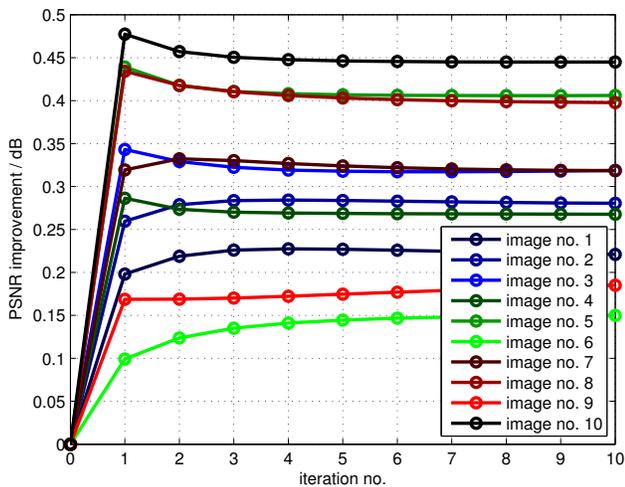}
\caption{Extension of the experiment shown in \fref{fig:psnr_gain}. The PSNR values do not always increase as more Sinkhorn-Knopp iterations are used. The curves are averaged over 20 independent Monte-Carlo trials with different noise realizations.}
\label{fig:sinkhorn_experiment}
\end{figure}

\begin{algorithm}[t]
\caption{Sinkhorn-Knopp Balancing Algorithm}
\begin{algorithmic}
\STATE Input: $\mW^{(0)}$
\WHILE{ $\| \mW^{(t+1)} - \mW^{(t)} \|_F > \mathtt{tol}$ }
    \STATE $\mD_c = \diag{ (\mW^{(t)})^T \vone}$ \hfill \% Column Normalize
    \STATE $\mD_r = \diag{ \mW^{(t)} \mD_c^{-1} \vone}$ \hfill \% Row Normalize
    \STATE $\mW^{(t+1)} = \mD_r^{-1}\mW^{(t)}\mD_c^{-1}$
\ENDWHILE
\end{algorithmic}
\label{alg:sinkhorn}
\end{algorithm}

Sinkhorn-Knopp has many interesting properties. First, when Sinkhorn-Knopp converges, the converging limit is a doubly-stochastic matrix --- a symmetric non-negative matrix with unit column and row (also called a \emph{symmetric} smoothing filter.) A doubly stochastic matrix has all of its eigenvalue's magnitudes bounded in $[0,1]$ so that repeated multiplications always attenuate the eigenvalues \cite{Baju_Hastie_Tibshirani_1989}. Moreover, the estimate $\vzhat$ formed by a doubly stochastic matrix is admissible in the sense that no other estimates are uniformly better \cite{Cohen_1966}.

To explain the performance gain, Milanfar \cite{Milanfar_2013b} considered a notion called ``effective degrees of freedom'', defined as
\begin{equation*}
\mathrm{df} = \sum_{j=1}^n \frac{\partial \widehat{z}_j}{\partial y_j},
\end{equation*}
where $\widehat{z}_j$ is the $j$th pixel of the estimate and $y_j$ is the $j$th pixel of the input. $\mathrm{df}$ measures how an estimator trades bias against variance. Larger values of $\mathrm{df}$ imply a lower bias but higher variance. Milanfar argued that the overall mean squared error, which is the sum of the bias and the variance, is reduced because one can prove that symmetric smoothing filters have high effective degrees of freedom. However, effective degrees of freedom is not easy to interpret. It will be more useful if we can geometrically describe the actions to which the column-row normalization are applying.

\subsection{Contributions}
The present paper is motivated by our wish to further understand the mechanism behind the performance gain phenomenon. Our approach is to study an Expectation-Maximization (EM) algorithm for learning a Gaussian mixture model (GMM.) By analyzing the E-step and the M-step of the EM algorithm, we find that the actions of the symmetrization is a type of data \emph{clustering}. This observation echoes with a number of recent work that shows ordering and grouping of non-local patches are key to high-quality image denoising \cite{Ram_Elad_Cohen_2014,Zoran_Weiss_2011,Taylor_Meyer_2012}. There are two contributions of this paper, described as follows.

First, we generalize the symmetrization process by reformulating the denoising problem as a maximum-a-posteriori (MAP) estimation under a Gaussian mixture model. We show that the original smoothing filter in \eref{eq:standard nlm}, the one-step Sinkhorn-Knopp in \eref{eq:sinkhorn one step}, and the full Sinkhorn-Knopp (i.e., iterative applications of Sinkhorn-Knopp until convergence) are all sub-routines of the EM algorithm to learn the GMM. By showing that each method supersedes its preceding counterpart, we provide a possible explanation for the performance gain phenomenon.

Second, based on the analysis of the GMM, we propose a new denoising algorithm called the GMM symmetric smoothing filter (GSF). We show that GSF does not only subsume a number of smoothing filters, but also has a performance similar to some state-of-the-art denoising algorithms. We will discuss implementation and parameter selections for the GSF algorithm.

This paper is an extension of a conference article presented in \cite{Chan_Zickler_Lu_2015}. In \cite{Chan_Zickler_Lu_2015}, the linkage between the GMM and the MAP denoising step was not thoroughly discussed. Specifically, the MAP was formulated as a weighted least squares step but the weights were indirectly obtained through a by-product of the EM algorithm. In this paper, we show that the quadratic cost function in the weighted least squares is indeed a surrogate function for solving the MAP problem. Therefore, minimizing the cost function of the weighted least squares is equivalent to minimizing an upper bound of the MAP objective function.

The rest of the paper is organized as follows. First, we provide a brief introduction to GMM and the EM algorithm in Section~\ref{sec:preliminary}. In Section~\ref{sec:generalization} we discuss the generalization of different symmetrizations using the EM algorithm. The new GSF is discussed in Section~\ref{sec:gm_nlm} and experimental results are shown in Section~\ref{sec:experiments}. We conclude in Section~\ref{sec:conclusion}.

\section{Preliminaries}
\label{sec:preliminary}

\subsection{Notations}
Throughout this paper, we use $n$ to denote the number of pixels in the noisy image, and $k$ to denote the number of mixture components in the GMM model. To avoid ambiguity, we call a mixture component of a GMM as a cluster. Clusters are tracked using the running index $i \in  \{1,\ldots,k\}$, whereas patches (or pixels) are tracked using the running index $j \in \{1,\ldots,n\}$. Without loss of generality, we assume that all pixel intensity values have been normalized to the range $[0, \, 1]$.

We use bold letters to denote vectors and the symbol $\vone$ to denote a constant vector of all ones. The vector $\vx_j \in \R^2$ represents the two-dimensional spatial coordinate of the $j$th pixel, and the vector $\vy_j \in \R^d$ represents a $d$-dimensional patch centered at the $j$th pixel of the noisy image $\vy$. For a clean image $\vz$, the $j$th patch is denoted by $\vz_j$. A scalar $y_j \in \R$ refers to the intensity value of the center pixel of $\vy_j$. Therefore, a $d$-dimensional patch $\vy_j$ (assume $d$ is an odd number) is a vector $\vy_j = [y_{j-(d-1)/2},\ldots,y_j,\ldots,y_{j+(d-1)/2}]^T$. To extract $\vy_j$ from the noisy image $\vy$, we use a linear operator $\mP_j \in \R^{d \times n}$ such that $\vy_j = \mP_j \vy$. For some smoothing filters, the spatial coordinate $\vx_j$ is used together with a patch $\vz_j$ (or $\vy_j$). Therefore, for generality we define a \emph{generalized patch} $\vp_j \in \R^p$, which could be $\vx_j$, $\vz_j$ (or $\vy_j$), or a concatenation of both: $\vp_j = [\vx_j^T, \; \vz_j^T]^T$ (or $\vp_j = [\vx_j^T, \; \vy_j^T]^T$.)

\begin{table}[t]
\setlength{\extrarowheight}{10pt}
\centering
\caption{Popular choices of the smoothing filter $\mW$. }
\begin{tabular}{|m{3cm}|c|}
\hline\hline
Filter & $W_{ij}$\\
\hline\hline
Gaussian Filter  \cite{Wand_Jones_1995}                & $\exp\left\{-\frac{\|\vx_j - \vx_i\|^2}{2h_s^2}\right\}$\\
\hline
Bilateral Filter \cite{Paris_Durand_2009}              & $\exp\left\{- \left(\frac{\|\vx_j - \vx_i\|^2}{2h_s^2}+\frac{(y_j - y_i)^2}{2h_r^2}\right)\right\}$\\
\hline
NLM \cite{Buades_Coll_2005_Journal}                    & $\exp\left\{-\frac{\|\vy_j - \vy_i\|^2}{2h_r^2}\right\}$\\
\hline
Spatially Regulated NLM \cite{Milanfar_2013b}                     & $\exp\left\{-\left(\frac{\|\vx_j - \vx_i\|^2}{2h_s^2}+\frac{\|\vy_j - \vy_i\|^2}{2h_r^2}\right)\right\}$\\
\hline
LARK \cite{Takeda_Farsiu_Milanfar_2007}                & $\exp\left\{-\frac{1}{2}(\vx_j - \vx_i)^T \mSigma_{i}^{-1}(\vx_j - \vx_i)\right\}$\\
\hline\hline
\end{tabular}
\label{table:choice of W}
\vspace{-2ex}
\end{table}

\subsection{Smoothing Filters}
The results presented in this paper are applicable to smoothing filters $\mW$ taking the following form:
\begin{equation}
W_{ij} = \kappa_i \; \calN( \vp_j \,|\, \vmu_i, \mSigma_i),
\label{eq:wij}
\end{equation}
where $\kappa_i \bydef \sqrt{(2\pi)^p |\mSigma_i|}$ is a normalization constant, and $\calN(\cdot)$ denotes a $p$-dimensional Gaussian:
\begin{align}
\calN( \vp_j \,|\, \vmu_i, \mSigma_i) \bydef \frac{1}{\kappa_i} \exp\left\{-\frac{1}{2}\|\vp_j-\vmu_i\|_{\mSigma_i^{-1}}^2\right\},
\label{eq:kernel}
\end{align}
with mean $\vmu_i \in \R^{p}$ and covariance matrix $\mSigma_i \in \R^{p \times p}$. We note that \eref{eq:wij} is general and covers a number of widely used filters as shown in Table \ref{table:choice of W}. 

\begin{example}[Standard NLM]
For the standard NLM \cite{Buades_Coll_2005_Journal}, we have $\vp_j = \vy_j$. The $i$th mean vector is $\vmu_i = \vy_i$, i.e., the noisy patch is the mean vector. The $i$th covariance matrix
is
\begin{equation}
\mSigma_i = h_r^2\mI, \quad\quad \mbox{for all } i,
\label{eq:mSigma}
\end{equation}
where $h_r$ is the NLM parameter. 

In practice, to reduce computational complexity, a search window $\Omega$ is often introduced so that neighboring patches are searched within $\Omega$. This is equivalent to multiplying an indicator function to the weight as
\begin{equation}
W_{ij} = \I\left\{ (\vx_i-\vx_j) \in \Omega \right\} \exp\left\{-\frac{\|\vy_j - \vy_i\|^2}{2h_r^2}\right\},
\label{eq:nlm indicator}
\end{equation}
where $\I\left\{ \vx \in \Omega \right\} = 1$ if $\vx \in \Omega$, and is zero otherwise. 
\end{example}

\begin{example}[Spatially Regulated NLM]
\label{example:nlm}
As an alternative to the hard thresholding introduced by the indicator function in \eref{eq:nlm indicator}, one can also consider the spatially regulated NLM \cite{Milanfar_2013b}. In this case, we can define $\vp_j = \begin{bmatrix} \vx_j^T, \; \vy_j^T \end{bmatrix}^T$. The mean vector $\vmu_i$ and the covariance matrices will consist of two parts:
\begin{equation}
\vmu_i =
\begin{bmatrix}
\vmu_i^{(s)} \\
\vmu_i^{(r)}
\end{bmatrix},
\quad\quad
\mSigma_i = \begin{bmatrix} h_s^2\mI & 0 \\ 0 & h_r^2\mI \end{bmatrix} \bydef \mSigma_{\mathrm{NLM}},
\label{eq:mSigma}
\end{equation}
where $h_s$ and $h_r$ are the spatial and the range parameters, respectively. This leads to the weight
\begin{equation}
W_{ij} = \exp\left\{-\frac{\|\vx_j - \vx_i\|^2}{2h_s^2}\right\} \exp\left\{-\frac{\|\vy_j - \vy_i\|^2}{2h_r^2}\right\}.
\label{eq:nlm s}
\end{equation}
It is not difficult to see that the spatially regulated NLM coincides with the standard NLM as $h_s \rightarrow \infty$ and $|\Omega| \rightarrow \infty$. In fact, the former uses a soft search window and the latter uses a hard search window. From our experience, we find that the spatially regulated NLM typically has better performance than the standard NLM when the best choices of $h_s$ and $\Omega$ are used in either case. Therefore, in the rest of this paper we will focus on the spatially regulated NLM.
\end{example}

\subsection{Gaussian Mixture Model}
The Gaussian mixture model (GMM) plays an important role in this paper. Consider a set of $n$ generalized patches $\calP \bydef \{\vp_1,\ldots,\vp_n\}$ where $\vp_j \in \R^p$. We say that $\calP$ is generated from a Gaussian mixture model of $k$ clusters if $\vp_j$ is sampled from the distribution
\begin{equation}
f(\vp_j \,|\, \mTheta) = \sum_{i=1}^k \pi_i \, \calN(\vp_j \,|\, \vmu_i,\mSigma_i),
\label{eq:gmm}
\end{equation}
where $\pi_i \in \R$ is the weight of the $i$th cluster, $\vmu_i \in \R^p$ is the mean vector, and $\mSigma_i \in \R^{p \times p}$ is the covariance matrix. For the purpose of analyzing smoothing filters in this paper, we shall assume that the covariance matrices $\mSigma_i$ are fixed according to the underlying smoothing filter. For example, in spatially regulated NLM we fix the covariance matrices as $\mSigma_i = \mSigma_{\mathrm{NLM}}$. When $\mSigma_i$'s are fixed, we denote $\mTheta \bydef \{\pi_i,\vmu_i\}_{i=1}^k$ as the GMM model parameters.

Learning the model parameters $\mTheta$ from $\calP$ is typically done using the Expectation-Maximization (EM) algorithm \cite{Gupta_Chen_2010}. The EM algorithm consists of two major steps: the expectation step (E-step) and the maximization step (M-step). The E-step is used to compute the conditional expected log-likelihood, often called the $Q$-function. The M-step is used to maximize the $Q$-function by seeking the optimal parameters $\mTheta$. The algorithm iterates until the log-likelihood converges. Since the EM algorithm is widely used, we skip the introduction and refer readers to \cite{Gupta_Chen_2010} for a comprehensive tutorial. The EM algorithm for learning a GMM is summarized in Algorithm~\ref{alg:em for gmm}.

\begin{algorithm}[t]
\caption{EM Algorithm for Learning a GMM with a known covariance matrix $\mSigma$ \cite{Gupta_Chen_2010}.}
\begin{algorithmic}
\STATE Input: Patches $\calP \bydef \{\vp_j\}_{j=1}^n$, and the number of clusters $k$.
\STATE Output: Parameters $\mTheta = \{(\pi_i,\vmu_i)\}_{i=1}^k$.
\STATE
\STATE Initialize $\pi_i^{(0)}$, $\vmu_i^{(0)}$ for $i = 1,\ldots,k$, and set $t = 0$.
\WHILE{not converge}
    \STATE \textbf{E-step}: Compute, for $i = 1,\ldots,k$ and $j = 1,\ldots,n$
        \begin{equation}
            \gamma_{ij}^{(t)} = \frac{ \pi_i^{(t)} \calN(\vp_j \,|\, \vmu_i^{(t)}, \mSigma) }{ \sum_{l=1}^k \pi_l^{(t)} \calN(\vp_j \,|\, \vmu_l^{(t)}, \mSigma) }
            \label{eq:em,gamma}
        \end{equation}
    \STATE \textbf{M-step}: Compute, for $i = 1,\ldots,k$
        \begin{align}
            \pi_i^{(t+1)}     &= \frac{ 1 }{ n } \sum_{j=1}^n \gamma_{ij}^{(t)} \label{eq:em,pi}\\
            \vmu_i^{(t+1)}    &= \frac{ \sum_{j=1}^n \gamma_{ij}^{(t)} \vp_j}{\sum_{j=1}^n \gamma_{ij}^{(t)}} \label{eq:em,mu}
        \end{align}
    \STATE \textbf{Update Counter}: $t \leftarrow t + 1$.
\ENDWHILE
\end{algorithmic}
\label{alg:em for gmm}
\end{algorithm}

\section{Generalizations of Symmetric Filters}
\label{sec:generalization}
In this section, we discuss how various symmetric smoothing filters are generalized by the EM algorithm. We begin by discussing how the GMM can be used for denoising.

\subsection{MAP Denoising Using GMM}
We first specify the denoising algorithm. Given the noisy image $\vy$, we formulate the denoising problem by using the maximum-a-posterior (MAP) approach:
\begin{align}
\vzhat
= \argmin{\vz} \;\; \frac{\lambda}{2} \|\vz - \vy\|^2 - \sum_{j=1}^n \log f(\vp_j \,|\, \mTheta)
\label{eq:epll}
\end{align}
where the first term specifies the data fidelity with $\lambda$ as a parameter. The second term, which is a sum of overlapping patches (thus are dependent), is called the expected patch log-likelihood (EPLL) \cite{Zoran_Weiss_2011}. Note that EPLL is a general prior that uses the expectation of the log-likelihood of overlapping patches. It is not limited to a particular distribution for $f(\vp_j \,|\, \mTheta)$, although in \cite{Zoran_Weiss_2011} the GMM was used. Note also that the patch $\vp_j$ in \eref{eq:epll} is extracted from the optimization variable $\vz$. Thus, by minimizing over $\vz$ we also minimize over $\vp_j$.

Substituting \eref{eq:gmm} into \eref{eq:epll}, we obtain a GMM-based MAP denoising formulation
\begin{align}
\vzhat = \argmin{\vz} \;\; \frac{\lambda}{2} \|\vz - \vy\|^2 - \sum_{j=1}^n \log \left(\sum_{i=1}^k \pi_i \calN(\vp_j \,|\, \vmu_i,\mSigma)\right).
\label{eq:epll gmm}
\end{align}
The computational challenge of \eref{eq:epll gmm} is the sum of exponentials inside the logarithm, which hinders closed-form solutions. In \cite{Zoran_Weiss_2011}, Zoran and Weiss proposed to use a half quadratic splitting method to alternatingly minimize a sequence of subproblems, and select the mode of the GMM in each subproblem.

Our solution to handle the optimization problem in \eref{eq:epll gmm} is to use a surrogate function under the Majorization-Maximization framework \cite{Hunter_Lange_2004}. The idea is to find a surrogate function $h(\vp_j, \, \vp_j' \,|\, \mTheta)$ such that
\begin{align}
h(\vp_j, \, \vp_j' \,|\, \mTheta)  &\ge -\log f(\vp_j \,|\, \mTheta), \quad \forall (\vp_j,\vp_j'), \label{eq:h condition1}\\
h(\vp_j, \, \vp_j  \,|\, \mTheta)  &=   -\log f(\vp_j \,|\, \mTheta). \label{eq:h condition2}
\end{align}
If we can find such function $h(\vp_j, \, \vp_j' \,|\, \mTheta)$, then the minimization in \eref{eq:epll} can be relaxed to minimizing an upper bound
\begin{align}
\vzhat = \argmin{\vz} \;\; \frac{\lambda}{2} \|\vz - \vy\|^2 + \sum_{j=1}^n h(\vp_j, \, \vp_j'  \,|\, \mTheta). \label{eq: surrogate minimization}
\end{align}
For the GMM in \eref{eq:gmm}, Zhang \emph{et al.} \cite{Zhang_Ye_Pal_2016} proposed one convenient surrogate function $h(\vp_j, \, \vp_j' \,|\, \mTheta)$ whose expression is given in Lemma~\ref{lemma:surrogate}.
\begin{lemma}[Surrogate function for GMM]
\label{lemma:surrogate}
The function
\begin{align}
&h(\vp_j, \, \vp_j' \,|\, \mTheta)
\bydef -\log f(\vp_j' \,|\, \mTheta) \label{eq: h}\\
&\quad\quad + \sum_{i=1}^k \gamma_{ij}\left( -\frac{1}{2}\|\vp_j-\vmu_i\|^2_{\mSigma^{-1}} + \frac{1}{2}\|\vp_j'-\vmu_i\|^2_{\mSigma^{-1}}\right), \notag
\end{align}
where
\begin{equation}
\gamma_{ij} = \frac{ \pi_i \calN(\vp_j' \,|\, \vmu_i, \mSigma) }{ \sum_{l=1}^k \pi_l \calN(\vp_j' \,|\, \vmu_l, \mSigma) },
\label{eq:gamma}
\end{equation}
is a surrogate function that satisfies \eref{eq:h condition1} and \eref{eq:h condition2}.
\end{lemma}
\begin{proof}
See Appendix A of \cite{Zhang_Ye_Pal_2016}.
\end{proof}

\begin{remark}
The surrogate function $h(\vp_j, \, \vp_j' \,|\, \mTheta)$ requires an intermediate variable $\vp_j'$. This variable can be chosen as $\vp_j' = [\vx_j^T, \vy_j^T]^T$, i.e., the generalized patch using the noisy image $\vy$. Thus, $\gamma_{ij}$ is independent of  $\vp_j$.
\end{remark}

Substituting the result of Lemma~\ref{lemma:surrogate} into \eref{eq: surrogate minimization}, we observe that \eref{eq: surrogate minimization} can be rewritten as
\begin{align}
\vzhat = \argmin{\vz} \;\; \lambda \|\vz - \vy\|^2 + \sum_{j=1}^n \sum_{i=1}^k \gamma_{ij} \|\vp_j-\vmu_i\|^2_{\mSigma^{-1}},
\label{eq:gsf minimization 0}
\end{align}
where we dropped terms involving $\vp_j'$ as they are independent to the optimization variable $\vz$. Note that \eref{eq:gsf minimization 0} is a quadratic minimization, which is significantly easier than \eref{eq:epll gmm}.

For the case of spatially regulated NLM, it is possible to further simplify \eref{eq:gsf minimization 0} by recognizing that $\vp_j$ involves both spatial coordinate $\vx_j$ and patch $\vz_j$. Therefore, by expressing $\vp_j = [\vx_j^T, \vz_j^T]^T$ and by defining $\vmu_i$ using \eref{eq:mSigma}, \eref{eq:gsf minimization 0} becomes
\begin{align*}
\vzhat
&= \argmin{\vz} \; \lambda \|\vz - \vy\|^2 + \sum_{j=1}^n \sum_{i=1}^k \gamma_{ij} \left\| \begin{bmatrix} \vx_j \\ \vz_j \end{bmatrix} - \begin{bmatrix} \vmu_i^{(s)} \\ \vmu_i^{(r)} \end{bmatrix}\right\|^2_{\mSigma^{-1}}.
\end{align*}
In this minimization, we observe that since $\vx_j$ is not an optimization variable, it can be eliminated without changing the objective function. By using a patch extract operator $\mP_j$, i.e., $\vz_j = \mP_j\vz$, we obtain the minimization
\begin{align}
\vzhat
&= \argmin{\vz} \; \lambda \|\vz - \vy\|^2 + \sum_{j=1}^n \sum_{i=1}^k \gamma_{ij} \left\| \mP_j \vz -  \vmu_i^{(r)} \right\|^2,
\tag{$P_1$}
\label{eq:gsf minimization}
\end{align}
where we have absorbed the NLM parameter $2h_r^2$ (the diagonal term of $\mSigma$) into $\lambda$. Problem \eref{eq:gsf minimization} is the main optimization of interest in this paper. We call \eref{eq:gsf minimization} the \emph{GMM Symmetric Smoothing Filter (GSF)}. In the literature, there are other GMM based denoising algorithms. Their connections to GSF will be discussed in Section IV .

\subsection{Original Smoothing Filter}
We now discuss the role of symmetrization by studying \eref{eq:gsf minimization} and the EM algorithm for learning the GMM. To keep track of the iterations of the EM algorithm, we use the running index $t$ denotes the iteration number. For consistency, we will focus on the spatially regulated NLM.

In spatially regulated NLM, the $i$th pixel of the denoised image is
\begin{equation}
\widehat{z}_i = \frac{\sum_{j=1}^n W_{ij}y_j}{\sum_{j=1}^n W_{ij}},
\label{eq:nlm solution}
\end{equation}
with $W_{ij}$ defined in \eref{eq:nlm s}. To obtain \eref{eq:nlm solution} from \eref{eq:gsf minimization}, we consider the following choices of parameters
\begin{align}
\mP_j\vz    = z_j, \;\; \vmu_i^{(r)} &= y_i, \;\; \pi_i = 1/k, \;\; \lambda = 0, \;\; k = n, \notag \\
\gamma_{ij} &= \calN(\vp_j \,|\, \vp_i, \mSigma_{\mathrm{NLM}}).
\label{eq:conditions original}
\end{align}
In this case, since the quadratic objective in \eref{eq:gsf minimization} is separable, the $i$th term becomes
\begin{equation}
\zhat_i = \argmin{z_i} \;\; \sum_{j=1}^n  \gamma_{ij} (z_i - y_j)^2 = \frac{\sum_{j=1}^n \gamma_{ij}y_j}{\sum_{j=1}^n \gamma_{ij}},
\end{equation}
which coincides with \eref{eq:nlm solution} as $W_{ij} = \calN(\vp_j \,|\, \vp_i, \mSigma_{\mathrm{NLM}})$ because of \eref{eq:nlm s}.

It is important to study the conditions in \eref{eq:conditions original}. First, the patch extractor $\mP_j$ extracts a pixel $z_j$ from $\vz$. This step is necessary because NLM is a weighted average of individual pixels, even though the weights are calculated using patches. Accordingly, the mean vector $\vmu_i$ is also a pixel $y_i$ so that it matches with the optimization variable $z_i$. From a clustering perspective, we can interpret $\vmu_i = y_i$ as having one cluster center for one pixel, i.e., every pixel has its own cluster. Clearly, this is a suboptimal configuration, and we will address it when we present the proposed method. We also note that in \eref{eq:conditions original}, the parameter $\lambda$ is $0$. What it means is that the data fidelity term is not used and only the EPLL prior is required to obtain the NLM result.

The most interesting observation in \eref{eq:conditions original} is the choice of $\gamma_{ij}$. In fact, the $\gamma_{ij}$ in \eref{eq:conditions original} is only the numerator of \eref{eq:gamma} by assuming $\pi_i = 1/k$. If we let $\mTheta_i = (\pi_i,\vmu_i)$ be the $i$th model parameter of a GMM, then the physical meaning of \eref{eq:conditions original} is a \emph{conditional probability} of observing $\vp_j$ given that we pick the $i$th cluster, i.e., $\Pb(\vp_j \,|\, \mTheta_i)$. In contrast, \eref{eq:gamma} is a \emph{posterior probability} of picking the $i$th cluster given that we observe $\vp_j$, i.e., $\Pb(\mTheta_i \,|\, \vp_j)$. We will discuss this subtle difference more carefully in the next subsection when we discuss the one-step Sinkhorn-Knopp.

\subsection{One-step Sinkhorn-Knopp}
In one-step Sinkhorn-Knopp, we recognize that the $i$th pixel of the denoised image is
\begin{equation}
 \widehat{z}_i = \frac{\sum_{j=1}^n \widetilde{W}_{ij}y_j}{\sum_{j=1}^n \widetilde{W}_{ij}} \quad\mbox{and}\quad \widetilde{W}_{ij} =  \frac{W_{ij}}{\sum_{l=1}^n W_{lj}}.
\end{equation}
This result can be obtained from \eref{eq:gsf minimization} by letting
\begin{align}
\mP_j\vz    = z_j, \;\; \vmu_i^{(r)} &= y_i, \;\; \pi_i = 1/k, \;\; \lambda = 0, \;\; k = n, \notag \\
\gamma_{ij} &= \frac{ \calN(\vp_j \,|\, \vp_i, \mSigma_{\mathrm{NLM}}) }{ \sum_{l=1}^k \calN(\vp_j \,|\, \vp_l, \mSigma_{\mathrm{NLM}}) }.
\label{eq:conditions one step}
\end{align}
Unlike the conditions posted by the original smoothing filter in \eref{eq:conditions original}, one-step Sinkhorn-Knopp defines $\gamma_{ij}$ according to \eref{eq:gamma}, or with some substitutions that yields \eref{eq:conditions one step}.

As mentioned in the previous subsection, the $\gamma_{ij}$ defined in \eref{eq:conditions original} is a conditional probability whereas that in \eref{eq:conditions one step} is a posterior probability. The posterior probability, if we refer to the EM algorithm (see Algorithm~\ref{alg:em for gmm}), is in fact the E-step with initializations $\pi_i^{(0)} = 1/k$ and $\vmu_i^{(0)} = \vp_i$. For the original filter, the E-step is not executed because $\gamma_{ij}$ is defined through \eref{eq:conditions original}. Therefore, from a clustering perspective, it is reasonable to expect a better denoising result from one-step Sinkhorn-Knopp than the original smoothing filter because the clustering is better performed.

To further investigate the difference between the conditional probability $\Pb(\vp_j \,|\, \mTheta_i)$ and the posterior probability $\Pb(\mTheta_i \,|\, \vp_j)$, we adopt a graph perspective by treating each patch $\vp_j$ as a node, and $\gamma_{ij}$ as the weight on the edge linking node $i$ and node $j$  \cite{Milanfar_2013a}. In the original smoothing filter, the conditional probability $\Pb(\vp_j \,|\, \mTheta_i)$ causes a ``majority vote'' effect, meaning that the parameter $\mTheta_i$ has a direct influence to every patch $\{\vp_j\}$ in the image. Therefore, if $\mTheta_i$ has many ``weak friends'', the sum of these ``weak friends'' can possibly alter the denoising result which would have been better obtained from a few ``good friends''.

In contrast, the one-step Sinkhorn-Knopp uses the posterior probability $\Pb(\mTheta_i \,|\, \vp_j)$. From Bayes rule, the posterior probability is related to the conditional probability by
$$
\Pb(\mTheta_i \,|\, \vp_j) = \frac{\Pb(\vp_j \,|\, \mTheta_i)\Pb(\mTheta_i)}{\Pb(\vp_j)}.
$$
Since $\Pb(\mTheta_i) = \pi_i$ and $\pi_i = 1/k$, we see that $\Pb(\mTheta_i \,|\, \vp_j)$ is the ratio of $\Pb(\vp_j \,|\, \mTheta_i)$ and $\Pb(\vp_j)$. $\Pb(\vp_j)$ measures the popularity of $\vp_j$. Thus, if $\vp_j$ is a popular patch (i.e., it is a ``friend'' of many), then the normalization $\Pb(\vp_j \,|\, \mTheta_i)/\Pb(\vp_j)$ is a way to balance out the influence of $\vp_j$. Interpreted in another way, it is equivalent to say that if $\mTheta_i$ has many ``weak friends'', the influence of these ``weak friends'' should be reduced. Such intuition is coherent to many NLM methods that attempt to limit the number of nearby patches, e.g., \cite{Kervrann_Boulanger_2008,Luo_Chan_Pan_2013}.

The definite answer to whether there is performance gain due to one-step Sinkhorn-Knopp is determined by the likelihood of obtaining ``weak friends''. This, in turn, is determined by the image content and the NLM parameter $h_r$ which controls the easiness of claiming ``friends''. For example, if an image contains many textures of a variety of content and if $h_r$ is large, then it is quite likely to obtain ``weak friends''. In this case, one-step Sinkhorn-Knopp will improve the denoising performance. On the other hand, if an image contains only a constant foreground and a constant background, then most patches are ``good friends'' already. Applying the one-step Sinkhorn-Knopp could possibly hurt the denoising performance. \fref{fig:PSNR gain negative} shows an example.

\begin{figure}[h]
\centering
\begin{tabular}{c}
\includegraphics[width=\linewidth]{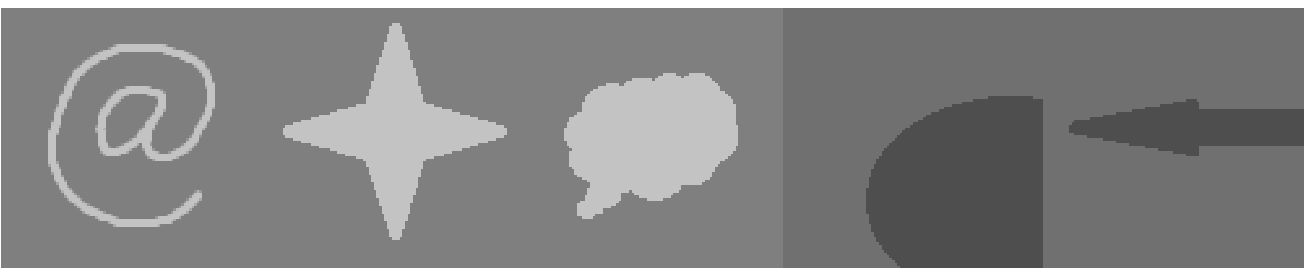}\\
\small{
\begin{tabular}{cccccc}
Image No. & 11 & 12 & 13 & 14 & 15\\
\hline
\hline
Original  & 33.48 & 35.15 & 36.39 & 38.85 & 36.81 \\
One-Step  & 32.98 & 34.67 & 35.59 & 38.36 & 36.43 \\
\hline
PSNR Gain & -0.50 & -0.47 & -0.80 & -0.48 & -0.38 \\
\hline
\end{tabular}
}
\end{tabular}
\caption{Repeat of the experiment in \fref{fig:psnr_gain} using images with a constant foreground and a constant background. Optimal $h_r$ is tuned for the original standard NLM and the one-step Sinkhorn-Knopp, respectively. Note that the PSNR gain are all negative.}
\label{fig:PSNR gain negative}
\end{figure}

\begin{figure}[h]
\centering
\includegraphics[width=0.9\linewidth]{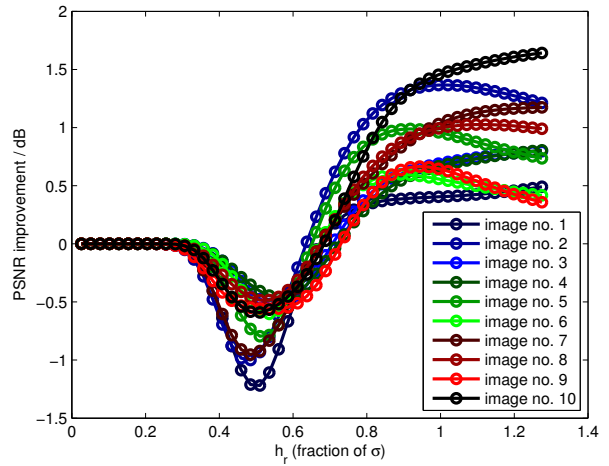}
\caption{Repeat the experiment in \fref{fig:psnr_gain} by plotting PSNR gain as a function of $h_r$. Note the consistent change of PSNR gain from negative to positive as $h_r$ increases.}
\label{fig:PSNR gain hr}
\end{figure}

To further justify the above claim that the performance gain is caused by the likelihood of obtaining ``weak friends'', we consider the 10 images in \fref{fig:psnr_gain} by plotting the PSNR gain as a function of $h_r$. Our hypothesis is that for small $h_r$, the performance gain should be small or even negative because it is difficult for a patch to find its ``friends''. When $h_r$ is large, the performance gain should become significant because many ``weak friends'' will be balanced out by the one-step Sinkhorn-Knopp. As shown in \fref{fig:PSNR gain hr}, this is in fact the case: For $h_r$ lies between 0 and certain threshold (around $0.65\sigma$ where $\sigma$ is the noise standard deviation), PSNR gain is always zero or negative. When $h_r$ increases, PSNR gain becomes positive. The result is consistent for all 10 images we tested.

\subsection{Full Sinkhorn-Knopp}
The full Sinkhorn-Knopp (Algorithm~\ref{alg:sinkhorn}) is an iterative algorithm that repeatedly applies the one-step Sinkhorn-Knopp until convergence. To analyze the full Sinkhorn-Knopp algorithm, we first recognize that the algorithm can be effectively described by two steps:
\begin{equation}
W_{ij} \leftarrow  \frac{\widetilde{W}_{ij}}{\sum_{l=1}^n \widetilde{W}_{il}} \quad\mbox{and}\quad \widetilde{W}_{ij} \leftarrow  \frac{W_{ij}}{\sum_{l=1}^n W_{lj}},
\label{eq:column row}
\end{equation}
where the first equation is a row normalization and the second equation is a column normalization. This pair of normalization can be linked to the EM algorithm in the following sense.

First, in the EM algorithm we fix the mixture weight $\pi_i^{(t)}$ as $\pi_i^{(t)} = 1/k$ for all clusters $i = 1,\ldots,k$ and all iteration numbers $t = 1,\ldots,t_{\mathrm{max}}$. This step is necessary because the full Sinkhorn-Knopp does not involve mixture weights. Intuitively, setting $\pi_i^{(t)} = 1/k$ ensures that all clusters have equal probability to be selected.

\begin{table*}[t]
\centering
\caption{Generalization and comparisons using EM algorithm for learning GMM with a known $\mSigma$.}
\scalebox{0.95}{
\begin{tabular}{|c|c|cccc|}
\hline
                    &                               & Original & One-step         & Full & GSF \\
                    &                               & Filter \cite{Buades_Coll_2005_Journal}   & Sinkhorn-Knopp \cite{Chan_Zickler_Lu_2013} & Sinkhorn-Knopp \cite{Sinkhorn_Knopp_1967}      & (Proposed)\\
\hline\hline
No. Clusters        &                               & $k = n$         & $k = n$                    & $k = n$                    & $k < n$\\
                    &                               &                 &                            &                            & (cross-validation)\\
\hline
Initialization      & $\gamma_{ij}^{(0)}$           & $\calN(\vp_j|\vmu_i^{(0)},\mSigmaNLM)$
                                                    & N/A
                                                    & N/A
                                                    & N/A\\
                    & $\pi_{i}^{(0)}$               & $1/k$           & $1/k$                      & $1/k$                      & $1/k$ \\
                    & $\vmu_i^{(0)}$                & $\vp_i$         & $\vp_i$                    & $\vp_i$                    & randomly picked $\vp_i$ \\
\hline
E-step              & Update $\gamma_{ij}^{(t)}$    & $\times$        & $\checkmark$           & $\checkmark$                & $\checkmark$\\
M-step              & Update $\pi_i^{(t)}$          & $\times$        & $\times$                  & $\times$                    & $\checkmark$\\
                    & Update $\vmu_i^{(t)}$         & $\times$        & $\times$                  & $\checkmark$ (implicitly)   & $\checkmark$\\
\hline
No. Iterations      &                               & 0               & 1                         & Many                        & Many \\
\hline
\hline
Denoising           & $\lambda$                     & 0               & 0                         & 0                           & by SURE \\
Parameters          & $\mP_j\vz$                    & $z_j$           & $z_j$                     & $z_j$                       & $\vp_j$\\
\hline
\end{tabular}
}
\label{table:gmm cases}
\end{table*}

When $\pi_i^{(t)}$ is fixed, the EM algorithm in Algorithm~\ref{alg:em for gmm} has only two steps: Update of $\gamma_{ij}^{(t)}$ in \eref{eq:em,gamma} and update of $\vmu_i^{(t)}$ in \eref{eq:em,mu}. Inspecting this pair of equations, we observe that \eref{eq:em,gamma} appears a column normalization whereas \eref{eq:em,mu} appears a row normalization. However, since full Sinkhorn-Knopp does not have a mean vector $\vmu_i^{(t)}$, we have to modify the EM algorithm in order to link the two. To this end, we modify \eref{eq:em,mu} by defining a sequence $\beta_{ij}^{(t)}$ such that
\begin{equation}
 \textbf{M-step}: \quad\quad\quad \beta_{ij}^{(t)}  = \frac{\gamma_{ij}^{(t)}}{\sum_{l=1}^n \gamma_{il}^{(t)}},
\label{eq:derive,sinkhorn,eq1a}
\end{equation}
and modify \eref{eq:em,gamma} by updating of $\gamma_{ij}^{(t)}$ via
\begin{equation}
\textbf{E-step}: \quad\quad\quad \gamma_{ij}^{(t)}  = \frac{ \beta_{ij}^{(t)}}{\sum_{l=1}^k  \beta_{lj}^{(t)}}.
\label{eq:derive,sinkhorn,eq1b}
\end{equation}
Under this setting, \eref{eq:derive,sinkhorn,eq1a}-\eref{eq:derive,sinkhorn,eq1b} becomes exactly \eref{eq:column row}.

It is important to understand the difference between the original EM algorithm using \eref{eq:em,gamma}-\eref{eq:em,mu} and the modified EM algorithm using \eref{eq:derive,sinkhorn,eq1a}-\eref{eq:derive,sinkhorn,eq1b}. In the M-step of the modified EM algorithm, the mean vector $\vmu_i^{(t)}$ is absent. However, $\vmu_i^{(t)}$ is still updated, though implicitly, because
\begin{align}
\vmu_i^{(t)} = \frac{\sum_{j=1}^n \gamma_{ij}^{(t)} \vp_j }{\sum_{l=1}^n \gamma_{il}^{(t)}} \overset{(a)}{=} \sum_{j=1}^n \beta_{ij}^{(t)}\vp_j,
\label{eq:derive,sinkhorn,eq4}
\end{align}
where $(a)$ follows from \eref{eq:derive,sinkhorn,eq1a}. Therefore, $\beta_{ij}^{(t)}$ are the coefficients that form $\vmu_i^{(t)}$ through a linear combination of $\vp_j$'s.

In the E-step of the modified EM algorithm, since $\vmu_i^{(t)}$ is absent, one cannot compute the conditional probability $\calN(\vp_j \,|\, \vmu_i^{(t)}, \mSigma)$ and hence the posterior probability in \eref{eq:em,gamma}. To resolve this issue, we recognize that since $\sum_{j=1}^n \beta_{ij}^{(t)} = 1$ and $\beta_{ij}^{(t)} \ge 0$ by definition, one possibility is to replace $\calN(\vp_j \,|\, \vmu_i^{(t)}, \mSigma)$ by $\beta_{ij}^{(t)}$. Such approximation is physically interpretable because $\beta_{ij}^{(t)}$ is the coefficient for the mean vector as shown in \eref{eq:derive,sinkhorn,eq4}. Thus, $\beta_{ij}^{(t)}$ having larger values will have a larger contribution to forming $\vmu_i^{(t)}$. In this perspective, $\beta_{ij}^{(t)}$ is performing a similar role as $\calN(\vp_j \,|\, \vmu_i^{(t)}, \mSigma)$. Practically, we observe that $\beta_{ij}^{(t)}$ is a good approximation when there are only a few distinctive clusters. This can be understood as that while individual $\beta_{ij}^{(t)}$ may not be accurate, averaging within a few large clusters can reduce the discrepancy between $\beta_{ij}^{(t)}$ and $\calN(\vp_j \,|\, \vmu_i^{(t)}, \mSigma)$.

The fact that the full Sinkhorn-Knopp does not resemble a complete EM algorithm offers some insights into the performance gain phenomenon. Recall from the results in \fref{fig:sinkhorn_experiment}, we observe that the first Sinkhorn-Knopp iteration always increases the PSNR except the artificial images we discussed in Section~\ref{sec:generalization}.C. This can be interpreted as that the clustering is properly performed by the EM algorithm. However, as more Sinkhorn-Knopp iterations are performed, some images show reduction in PSNR, e.g., images 3, 4, 8, 10. A close look at these images suggests that they contain complex texture regions that are difficult to form few but distinctive clusters. In this case, the approximation of $\calN(\vp_j \,|\, \vmu_i^{(t)}, \mSigma)$ by $\beta_{ij}^{(t)}$ is weak and hence the denoising performance drops.

\subsection{Summary}
To summarize our findings, we observe that the performance of the normalization is related to how the EM algorithm is being implemented. A summary of these findings is shown in Table~\ref{table:gmm cases}. For all the three algorithms we considered: the original filter, the one-step Sinkhorn-Knopp, and the full Sinkhorn-Knopp algorithm, the EM algorithm is not completely performed or in-properly configured. For example, setting $k = n$ causes excessive number of clusters and should be modified to $k < n$; the MAP parameter $\lambda$ is always 0 and should be changed to a positive value to utilize the data fidelity term in \eref{eq:gsf minimization}; the E-step and the M-step are not performed as it should be. Therefore, in the following section we will propose a new algorithm that completely utilizes the EM steps for problem \eref{eq:gsf minimization}. 
\section{GMM Symmetric Smoothing Filter}
\label{sec:gm_nlm}
The proposed denoising algorithm is called the Gaussian Mixture Model Symmetric Smoothing Filter (GSF). The overall algorithm of GSF consists of two steps:
\begin{itemize}
\item Step 1: Estimate the GMM parameter $\vmu_i^{(r)}$ and $\gamma_{ij}$ from the noisy image the by EM algorithm.
\item Step 2: Solve Problem \eref{eq:gsf minimization}, which has a closed form solution.
\end{itemize}
In the followings we discuss how GSF is implemented.

\subsection{Closed Form Solution of GSF}
First of all, we recall that since \eref{eq:gsf minimization} is a quadratic minimization, it is possible to derive a closed form solution by considering the first order optimality condition, which yields a normal equation
\begin{equation}
\left(\sum_{j=1}^n \mP_j^T\mP_j + \lambda \mI \right)\vzhat = \sum_{j=1}^n \mP_j^T\vw_j + \lambda \vy,
\label{eq:normal eq}
\end{equation}
where the vector $\vw_j$ is defined as
\begin{equation}
\vw_j \bydef \sum_{i=1}^n \gamma_{ij}\vmu_i^{(r)}.
\label{eq:vwj}
\end{equation}
Equations \eref{eq:normal eq}-\eref{eq:vwj} has a simple interpretation: The intermediate vector $\vw_j$ is a weighted average of the mean vectors $\{\vmu_i^{(r)}\}_{i=1}^k$. These $\{\vw_j\}_{j=1}^n$ represent a collection of (denoised) overlapping patches. The operation $\mP_j^T$ on the right hand side of \eref{eq:normal eq} aggregates these overlapping patches, similar to the aggregation step in BM3D \cite{Dabov_Foi_Katkovnik_2007}. The addition of $\lambda \vy$ regulates the final estimate by adding a small amount of fine features, depending on the magnitude of $\lambda$.

In order to use \eref{eq:normal eq}-\eref{eq:vwj}, we must resolve two technical issues related to the EM algorithm and Problem \eref{eq:gsf minimization}: (i) How to determine $\lambda$; (ii) How to determine $k$.

\subsection{Parameter $\lambda$}
Ideally, $\lambda$ should be chosen as the one that minimizes the mean squared error (MSE) of the denoised image. However, in the absence of the ground truth, MSE cannot be calculated directly. To alleviate this difficulty, we consider the Stein's Unbiased Risk Estimator (SURE) \cite{Stein_1981,Ramani_Blu_Unser_2008}. SURE is a consistent and unbiased estimator of the MSE. That is, SURE converges to the true MSE as the number of observations grows. Therefore, when there are sufficient number of observed pixels (which is typically true for images), minimizing the SURE is equivalent to minimizing the true MSE.

In order to derive SURE for our problem, we make an assumption about the boundary effect of $\mP_j$.
\begin{assumption}
\label{assumption: P}
We assume that the patch-extract operator $\{\mP_j\}_{j=1}^n$ satisfies the following approximation:
\begin{equation}
\sum_{j=1}^n \mP_j^T\mP_j = d \mI.
\label{eq:sum P}
\end{equation}
\end{assumption}
\noindent We note that Assumption \ref{assumption: P} only affects the boundary pixels and not the interior pixels. Intuitively, what Assumption \ref{assumption: P} does is to require that the boundary pixels of the image are periodically padded instead of zero-padded. In the image restoration literature, periodic boundary padding is common when analyzing deblurring methods, e.g., \cite{Michaeli_Irani_2014}.

\begin{figure}[t]
\centering
\includegraphics[width=0.95\linewidth]{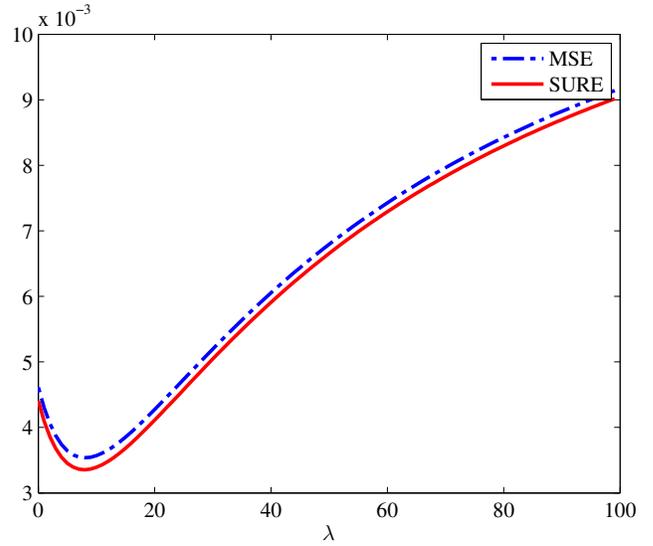}
\caption{Comparison between SURE and the ground truth MSE of a denoising problem.}
\label{fig:sure}
\end{figure}

Under Assumption \ref{assumption: P}, we can substitute \eref{eq:sum P} into \eref{eq:normal eq} and take the matrix inverse. This would yield
\begin{equation}
\vzhat(\lambda) = \frac{d}{d+\lambda}\vu + \frac{\lambda}{d+\lambda}\vy,
\label{eq:zhat lambda}
\end{equation}
where
\begin{equation}
\vu \bydef \frac{1}{d}\sum_{j=1}^n \mP_j^T \vw_j.
\label{eq:vu}
\end{equation}
Then, we can derive the SURE of $\vzhat$ as follows.

\begin{proposition}
Under Assumption \ref{assumption: P}, the SURE of $\vzhat(\lambda)$ is
\begin{equation}
\mathrm{SURE}(\lambda) = -\sigma^2 + \sigmahat^2\left( \frac{d}{d+\lambda}\right)^2 + \frac{2\sigma^2}{n}\left(\frac{\mathrm{div}(\vu) d + n\lambda}{d + \lambda}\right),
\label{eq:sure main}
\end{equation}
where $\sigmahat^2 \bydef \frac{1}{n}\|\vu - \vy\|^2$, and
\begin{align}
\mathrm{div}(\vu) &\bydef \vone_{n \times 1}^T  \left( \frac{1}{d}\sum_{j=1}^n\mP_j^T \left(\sum_{i=1}^k\gamma_{ij} \left(\frac{\sum_{j=1}^n \gamma_{ij} \ve_j}{\sum_{j=1}^n \gamma_{ij}}\right)\right)\right),
\label{eq:div vu}
\end{align}
where $\ve_j \in \R^d$ is the $j$th standard basis.
\end{proposition}
\begin{proof}
See Appendix B.
\end{proof}

The SURE given in \eref{eq:sure main} is a one-dimensional function in $\lambda$. The minimizer can be determined in closed-form.
\begin{corollary}
The optimal $\lambda$ that minimizes $\mathrm{SURE}(\lambda)$ is
\begin{equation}
\lambda^* = \max\left( d\left( \left(\frac{\sigmahat^2}{\sigma^2}\right)\left(\frac{n}{n-\mathrm{div}(\vu)}\right)-1 \right), \; 0 \right).
\end{equation}
\end{corollary}
\begin{proof}
\eref{eq:sure main} is a differentiable function in $\lambda$. Therefore, the minimizer can be determined by considering the first order optimality and set the derivative of $\mathrm{SURE}(\lambda)$ to zero. The projection operator $\max(\cdot,0)$ is placed to ensure that $\lambda^* \ge 0$.
\end{proof}

\begin{example}
To demonstrate the effectiveness of SURE, we show a typical MSE and a typical SURE curve of a denoising problem. In this example, we consider a $128 \times 128$ image ``Baboon'', with noise standard deviation of $\sigma = 30/255$. The non-local means parameters are $h_r = \sigma$ and $h_s = 10$. The number of clusters is $k = 50$, and the patch size is $5 \times 5$. The results are shown in \fref{fig:sure}, where we observe that the SURE curve and the true MSE curve are very similar. In fact, the minimizer of the true MSE is $\lambda = 8.0080$ with a PSNR of $24.5143$dB whereas the minimizer of SURE is $\lambda = 7.9145$ with a PSNR of $24.5141$dB.
\end{example}

\begin{remark}
Careful readers may notice that in \eref{eq:div vu}, we implicitly assume that $\gamma_{ij}$ is independent of $\vy_j$. This implicit assumption is generally not valid if $\gamma_{ij}$ is learned from $\vy$. However, in practice, we find that if we feed the EM algorithm with some initial estimate (e.g., by running the algorithm with $\lambda = 0$), then the dependence of $\gamma_{ij}$ from $\vy_j$ becomes negligible.
\end{remark}

\subsection{Number of Clusters $k$}
\label{sec:gm_nlm_cross}
The number of clusters $k$ is another important parameter. We estimate $k$ based on the concept of cross validation \cite{Wasserman_2005}.

Our proposed cross-validation method is based on comparing the estimated covariance with $\mSigmaNLM$. More specifically, we compute the estimated covariance
\begin{equation}
\mSigmahat_i = \frac{ \sum_{j=1}^n \gamma_{ij} \left(\vp_j - \vmu_i \right) \left(\vp_j - \vmu_i \right)^T}{ \sum_{j=1}^n \gamma_{ij} }, \label{eq:mSigmahat}
\end{equation}
where $\vmu_i = [\vmu_i^{(s)}, \, \vmu_i^{(r)}]^T$ is the mean returned by the EM algorithm, and $\gamma_{ij} = \pi_{ij}^{(\infty)}$ is the converged weight. Then, we compute the ratio of the deviation
\begin{equation}
\delta_i(k) = \frac{1}{d}\trace{\mSigmaNLM^{-1}\mSigmahat_i}.
\label{eq:deltai}
\end{equation}
Ideally, if $\mSigmahat_i = \mSigmaNLM$, then by \eref{eq:deltai} we have $\delta_i(k)= 1$. However, if the $i$th estimated Gaussian component has a radius significantly larger than $h_r$ (or, $h_s$ for the spatial components), then the covariance $\mSigmahat_i$ would deviate from $\mSigmaNLM$ and hence $\delta_i(k) > 1$. Conversely, if the $i$th estimated Gaussian component has a radius significantly smaller than $h_r$, then we will have $\delta_i(k) < 1$. Therefore, the goal of the cross validation is to find a $k$ such that $\delta_i(k)$ is close to 1.

To complete the cross-validation setup, we average $\delta_i(k)$ over all $k$ clusters to obtain an averaged ratio
\begin{equation}
\delta(k) = \frac{1}{k}\sum_{i=1}^k \delta_i(k).
\label{eq:deltak}
\end{equation}
The parenthesis $(k)$ in \eref{eq:deltak} emphasizes that both $\delta(k)$ and $\delta_i(k)$ are functions of $k$. With \eref{eq:deltak}, we seek the root $k$ of the equation $\delta(k) = 1$.

The root finding process for $\delta(k) = 1$ can be performed using the secant method. Secant method is an extension of the bisection method in which the bisection step size (i.e., $1/2$) is now replaced by an adaptive step size determined by the local derivative of the function. Let $k^{(a)}$ and $k^{(b)}$ be two number of clusters, and $\delta^{(a)}$ and $\delta^{(b)}$ be the corresponding cross-validation scores, i.e., $\delta^{(a)} = \delta(k^{(a)})$. If $\delta^{(a)} > 1$ and $\delta^{(b)}<1$, the secant method computes the new $k$ as
\begin{equation}
k^{(c)} = \frac{ k^{(a)}(\delta^{(b)}-1) - k^{(b)}(\delta^{(a)}-1)}{\delta^{(b)} - \delta^{(a)}}.
\label{eq:k m+1}
\end{equation}
If $\delta(k^{(c)}) > 1$, then we replace $k^{(a)}$ by $k^{(c)}$; Otherwise, we replace $k^{(b)}$ by $k^{(c)}$. The process repeats until the $|k^{(a)} - k^{(c)}| < \texttt{tol}$ and $|k^{(b)} - k^{(c)}| < \texttt{tol}$. A pictorial illustration of the secant method is shown in \fref{fig:secant}. A pseudo code is given in Algorithm \ref{alg:cross validation}.

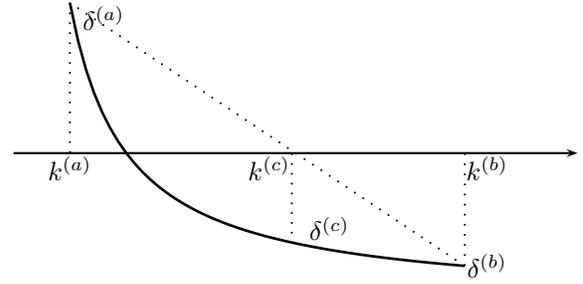
\begin{figure}[t]
\centering
\begin{pspicture}(-0.5,-1.8)(9.5,2)
\psline{->}(0,0)(7.5,0)
\psplot[algebraic,linewidth=1pt]{0.75}{6}{3/x-2}
\psline[linestyle=dotted](0.75,0)(0.75,2)
\psline[linestyle=dotted](6,0)(6,-1.5)
\psline[linestyle=dotted](6,-1.5)(0.75,2)
\psline[linestyle=dotted](3.7,0)(3.7,-1.2)
\rput(0.75,-0.2){$k^{(a)}$}
\rput(6.3,-0.2){$k^{(b)}$}
\rput(1.2,1.8){$\delta^{(a)}$}
\rput(6.3,-1.5){$\delta^{(b)}$}
\rput(3.4,-0.2){$k^{(c)}$}
\rput(4.2,-1){$\delta^{(c)}$}
\end{pspicture}
\caption{Illustration of the secant method. Given $k^{(a)}$ and $k^{(b)}$, we compute $k^{(c)}$ according to the slope defined by the line linking $\delta^{(a)}$ and $\delta^{(b)}$.}
\label{fig:secant}
\end{figure}

\begin{algorithm}[t]
\caption{Cross Validation to Determine $k$}
\begin{algorithmic}
\STATE Input: $k^{(a)}$ and $k^{(b)}$ such that $\delta^{(a)} > 1$ and $\delta^{(b)} < 1$.
\STATE Output: $k^{(c)}$.
\STATE
\WHILE{$|k^{(a)} - k^{(c)}| > \texttt{tol}$ and $|k^{(b)} - k^{(c)}| > \texttt{tol}$}
\STATE Compute $k^{(c)}$ according to \eref{eq:k m+1}.
\STATE Compute $\delta^{(c)} \bydef \delta(k^{(c)})$ according to \eref{eq:deltak}.
\IF{$\delta(k^{(c)}) > 1$}
    \STATE $k^{(a)} \leftarrow k^{(c)}$; $\delta^{(a)} \leftarrow \delta^{(c)}$.
    \ELSE
    \STATE $k^{(b)} \leftarrow k^{(c)}$; $\delta^{(b)} \leftarrow \delta^{(c)}$.
\ENDIF
\ENDWHILE
\end{algorithmic}
\label{alg:cross validation}
\end{algorithm}

\begin{figure}[!]
\centering
\includegraphics[width=\linewidth]{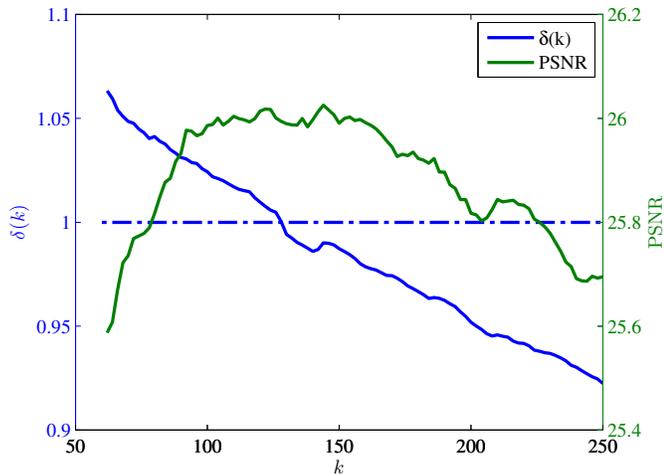}
\caption{Comparison between the cross validation score $\delta(k)$ and the true PSNR value as a function of $k$. The horizontal dashed line indicates the intersection at $\delta(k) = 1$.}
\label{fig:cross validation}
\end{figure}

\begin{example}
To verify the effectiveness of the proposed cross validation scheme, we consider a $128 \times 128$ ``House'' image with noise $\sigma = 60/255$. The patch size is $5 \times 5$, $h_r = \sigma$, and $h_s = 10$. \fref{fig:cross validation} shows the PSNR value of the denoised image and the corresponding cross validation score $\delta(k)$ as a function of $k$. For this experiment, the maximum PSNR is achieved at $k = 144$, where $\mathrm{PSNR}$ = 26.0257dB. Using the cross-validation score $\delta(k)$, we find that $\delta(k)$ is closest to $1$ when $k = 130$. The corresponding PSNR value is 25.9896dB, which is very similar to the true maximum PSNR.
\end{example}

\section{Experiments}
\label{sec:experiments}

In this section, we present additional simulation results to evaluate the proposed GSF.

\subsection{Experiment Settings}
We consider 10 testing images, each of which is resized to $128 \times 128$ (so $n = 16384$) for computational efficiency. The noise standard deviations are set as $\sigma \in \{20/255,\, 40/255,\, 60/255,\, 80/255,\, 100/255\}$. Several existing denoising algorithms are studied, namely the NLM \cite{Buades_Coll_2005_Journal}, One-step Sinkhorn-Knopp \cite{Chan_Zickler_Lu_2013}, BM3D \cite{Dabov_Foi_Katkovnik_2007}, EPLL \cite{Zoran_Weiss_2011}, Global image denoising (GLIDE) \cite{Talebi_Milanfar_2014}, NL-Bayes \cite{Lebrun_Buades_Morel_2013}, and PLE \cite{Yu_Sapiro_Mallat_2012}. The parameters of the methods are configured as shown in Table~\ref{table:configuration}.

For NLM and One-step Sinkhorn-Knopp (One-step, in short), we use the spatially regulated version due to its better performance over the standard NLM. We implement the algorithms by setting the patch size as $5 \times 5$ (i.e., $d=25$). The parameters are $h_s = 10$ and $h_r = \sigma \sqrt{d}$. The full Sinkhorn-Knopp algorithm is implemented using GLIDE \cite{Talebi_Milanfar_2014}, where the source code is downloaded from the author's website \footnote{GLIDE: \url{https://users.soe.ucsc.edu/~htalebi/GLIDE.php}}. Default settings of GLIDE are used in our experiment.

For the proposed GSF, we keep the same settings as NLM except for the intensity parameter $h_r$ where we set $h_r = \sigma$. The omission of the factor $\sqrt{d}$ is due to the fact that each Gaussian component is already a $d$-dimensional multivariate distribution. It is therefore not necessary to normalize the distance $\|\vy_i - \vy_j\|^2$ by the factor $d$.

For BM3D, EPLL, NL-Bayes, we downloaded the original source code from the author's website \footnote{BM3D: \url{http://www.cs.tut.fi/~foi/GCF-BM3D/}}$^,$\footnote{EPLL: \url{http://people.csail.mit.edu/danielzoran/}}$^,$\footnote{NL-Bayes: \url{http://www.ipol.im/pub/art/2013/16/}}. For PLE, we modified an inpainting version of the source code provided by the authors. Default settings of these algorithms are used.

Among these methods, we note that EPLL is an external denoising algorithm where a Gaussian mixture is learned from a collection of 2 million clean patches. All other methods (including GSF) are single image denoising algorithms.

\begin{table}[h]
\caption{Configurations of Methods}
\begin{tabular}{ll}
\hline
Method                                     & Configuration \\
\hline
NLM \cite{Buades_Coll_2005_Journal}        & Patch size $5 \times 5$, $h_s = 10$, $h_r = \sigma \sqrt{d}$\\
One-step \cite{Chan_Zickler_Lu_2013}       & Patch size $5 \times 5$, $h_s = 10$, $h_r = \sigma \sqrt{d}$\\
GSF (Ours)                                 & Patch size $5 \times 5$, $h_s = 10$, $h_r = \sigma$\\
GLIDE \cite{Talebi_Milanfar_2014}          & Default settings. Pilot estimate uses NLM.\\
NL-Bayes \cite{Lebrun_Buades_Morel_2013}   & Base mode. Default settings.\\
PLE \cite{Yu_Sapiro_Mallat_2012}           & Default settings. Default initializations.\\
BM3D \cite{Dabov_Foi_Katkovnik_2007}       & Default settings.\\
EPLL \cite{Zoran_Weiss_2011}               & Default settings. External Database.\\
\hline
\end{tabular}
\label{table:configuration}
\end{table}

\subsection{Comparison with NLM, One-step and Full Sinkhorn-Knopp}
The overall results of the experiment are shown in Table~\ref{table:standard_experiment}. We first compare the PSNR values of GSF with NLM, One-step and full Sinkhorn-Knopp.

In Table~\ref{table:improvement_sigma} we show the average PSNR over the 10 testing images. In this table, we observe that on average One-step has a higher PSNR than NLM by 0.12dB to 1.12dB, with more significant improvements at low noise levels. This implies that the ``grouping'' action by the column normalization becomes less influential when noise increases. Moreover, if we compare GSF with NLM and One-step, we observe that the PSNR gain is even larger. Even at a high noise level (e.g., $\sigma = 80/255$ or $\sigma = 100/255$), the average gain from NLM is 2.5dB or more.

\begin{table}[h]
\centering
\caption{PSNR comparison with different noise level $\sigma$. Results are averaged over 10 testing images. }
\small{
\begin{tabular}{cccccc}
\hline
            & \hspace{-0.5cm} NLM                       &  One-Step             &   Ours                 & $\mathrm{PSNR}_2 $ & $\mathrm{PSNR}_3$\\
$\sigma$    & \hspace{-0.4cm}  ($\mathrm{PSNR}_1$)      &  ($\mathrm{PSNR}_2$)  &  ($\mathrm{PSNR}_3$)   & $-\mathrm{PSNR}_1$ & $-\mathrm{PSNR}_1$\\
\hline
 20 	 & 	 25.59 	 & 	 26.71 	 & 	 28.89 	 & 	 +1.12 	 & 	 +3.30\\
 40 	 & 	 21.97 	 & 	 22.53 	 & 	 25.38 	 & 	 +0.56 	 & 	 +3.41\\
 60 	 & 	 20.33 	 & 	 20.63 	 & 	 23.60 	 & 	 +0.30 	 & 	 +3.27\\
 80 	 & 	 19.46 	 & 	 19.64 	 & 	 22.39 	 & 	 +0.18 	 & 	 +2.92\\
100 	 & 	 18.97 	 & 	 19.09 	 & 	 21.47 	 & 	 +0.12 	 & 	 +2.50\\
\hline
\end{tabular}
}
\label{table:improvement_sigma}
\end{table}

Besides studying the trend of PSNR as a function of $\sigma$, it is also interesting to compare the PSNR when we increase the spatial parameter $h_s$. In Table~\ref{table:improvement}, we show the PSNR improvement when we use different $h_s \in \{5,10,20,50,100\}$ for a $128 \times 128$ image. The results show that when $h_s$ increases, the PSNR improvement also increases. One reason is that in \eref{eq:mSigma}, the spatial parameter $h_s$ controls the diagonal bandwidth of the smoothing filter $\mW$. That is, a small $h_s$ leads to a banded diagonal $\mW$ with small bandwidth. In the limit when $h_s \rightarrow 0$, $\mW$ will become a diagonal matrix, and hence is immune to any column normalization. Therefore, the effectiveness of the column normalization in the One-step depends on how large $h_s$ is.

\begin{table}[h]
\centering
\caption{PSNR comparison with different parameter $h_s$. The testing image is ``Man''. $\sigma = 40/255$. }
\small{
\begin{tabular}{cccccc}
\hline
            & \hspace{-0.5cm} NLM     &  One-Step &   Ours    & $\mathrm{PSNR}_2 $ & $\mathrm{PSNR}_3$\\
$h_s$       & \hspace{-0.4cm}  ($\mathrm{PSNR}_1$)    &  ($\mathrm{PSNR}_2$)   &  ($\mathrm{PSNR}_3$)   & $-\mathrm{PSNR}_1$ & $-\mathrm{PSNR}_2$  \\
\hline
  5 	 & 	 \hspace{-0.4cm}22.82 	 & 	 23.08 	 & 	 24.76 	 & 	 +0.26 	 & 	 +1.68\\
 10 	 & 	 \hspace{-0.4cm}21.83 	 & 	 22.24 	 & 	 24.83 	 & 	 +0.41 	 & 	 +2.60\\
 20 	 & 	 \hspace{-0.4cm}21.25 	 & 	 21.66 	 & 	 24.79 	 & 	 +0.41 	 & 	 +3.13\\
 50 	 & 	 \hspace{-0.4cm}20.92 	 & 	 21.59 	 & 	 24.74 	 & 	 +0.68 	 & 	 +3.15\\
100 	 & 	 \hspace{-0.4cm}20.53 	 & 	 21.38 	 & 	 24.73 	 & 	 +0.85 	 & 	 +3.36 \\
\hline
\end{tabular}
}
\label{table:improvement}
\end{table}

\renewcommand{\arraystretch}{1.2}
\begin{table*}[!]
\centering
\tabcolsep=0.13cm
\caption{Denoising results of Standard NLM \cite{Buades_Coll_2005_Journal}, One-step Sinkhorn-Knopp \cite{Chan_Zickler_Lu_2013}, BM3D \cite{Dabov_Foi_Katkovnik_2007}, EPLL \cite{Zoran_Weiss_2011}, Global image denoising \cite{Talebi_Milanfar_2014}, and the proposed GSF.}
\footnotesize{
\begin{tabular}{|c|cccccccc|cccccccc|}
\hline\hline
         & NLM    &   OneStep  &  GSF    & GLIDE & Bayes &  PLE &  BM3D   &  EPLL
         & NLM    &   OneStep  &  GSF    & GLIDE & Bayes &  PLE &  BM3D   &  EPLL      \\
         & \cite{Buades_Coll_2005_Journal}  &  \cite{Chan_Zickler_Lu_2013}  & (ours) & \cite{Talebi_Milanfar_2014}    &  \cite{Lebrun_Buades_Morel_2013}&  \cite{Yu_Sapiro_Mallat_2012} & \cite{Dabov_Foi_Katkovnik_2007}     & \cite{Zoran_Weiss_2011}
         & \cite{Buades_Coll_2005_Journal}  &  \cite{Chan_Zickler_Lu_2013}  & (ours) & \cite{Talebi_Milanfar_2014}    &  \cite{Lebrun_Buades_Morel_2013}&  \cite{Yu_Sapiro_Mallat_2012} & \cite{Dabov_Foi_Katkovnik_2007}     & \cite{Zoran_Weiss_2011}   \\
\hline\hline
$\sigma$ & \multicolumn{8}{c|}{\emph{Baboon} } & \multicolumn{8}{c|}{\emph{Barbara} } \\
\hline
	20	&	24.53	&	25.01	&	26.84	&	26.51	&	27.22	&	26.14	&	26.96	&	27.19	&	26.05	&	27.02	&	29.43	&	28.64	&	29.52	&	28.82	&	29.42	&	29.40	\\
	40	&	22.32	&	22.55	&	24.49	&	24.04	&	24.55	&	23.75	&	24.57	&	24.56	&	21.48	&	21.91	&	25.41	&	24.83	&	25.60	&	24.74	&	25.35	&	25.79	\\
	60	&	21.31	&	21.46	&	23.32	&	22.87	&	23.03	&	22.67	&	23.53	&	23.44	&	19.31	&	19.55	&	23.28	&	22.33	&	23.61	&	22.54	&	23.55	&	23.64	\\
	80	&	20.76	&	20.87	&	22.51	&	22.22	&	22.21	&	21.72	&	22.77	&	22.66	&	18.25	&	18.38	&	21.83	&	20.92	&	22.11	&	21.05	&	22.30	&	22.11	\\
	100	&	20.43	&	20.52	&	21.98	&	20.68	&	21.80	&	20.49	&	22.14	&	22.09	&	17.68	&	17.76	&	20.77	&	19.82	&	20.84	&	19.83	&	21.30	&	21.00	\\
\hline	
$\sigma$ & \multicolumn{8}{c|}{\emph{Boat} } & \multicolumn{8}{c|}{\emph{Bridge} } \\
\hline
	20	&	24.88	&	26.05	&	28.43	&	27.55	&	28.59	&	27.53	&	28.58	&	28.76	&	23.99	&	24.95	&	26.90	&	26.34	&	27.26	&	26.83	&	27.09	&	27.25	\\
	40	&	21.97	&	22.39	&	24.96	&	24.35	&	25.04	&	24.26	&	25.12	&	25.32	&	20.87	&	21.35	&	23.85	&	23.18	&	24.02	&	23.29	&	23.88	&	24.19	\\
	60	&	20.46	&	20.70	&	23.19	&	22.59	&	23.60	&	22.21	&	23.47	&	23.56	&	19.58	&	19.85	&	22.24	&	21.47	&	22.45	&	21.37	&	22.44	&	22.48	\\
	80	&	19.60	&	19.74	&	22.14	&	21.42	&	22.21	&	21.19	&	22.43	&	22.41	&	18.83	&	19.01	&	21.20	&	20.44	&	21.25	&	20.35	&	21.45	&	21.42	\\
	100	&	19.09	&	19.18	&	21.34	&	20.50	&	21.43	&	19.81	&	21.74	&	21.53	&	18.35	&	18.48	&	20.46	&	19.75	&	20.16	&	19.79	&	20.67	&	20.66	\\
\hline																	
$\sigma$ & \multicolumn{8}{c|}{\emph{Couple} } & \multicolumn{8}{c|}{\emph{Hill} } \\
\hline
	20	&	24.54	&	25.62	&	28.20	&	27.25	&	28.33	&	27.43	&	28.42	&	28.60	&	25.51	&	26.38	&	28.68	&	27.98	&	28.99	&	28.12	&	28.82	&	28.97	\\
	40	&	21.67	&	22.10	&	24.64	&	23.95	&	24.96	&	23.86	&	25.00	&	25.11	&	22.58	&	23.11	&	25.55	&	24.79	&	25.61	&	24.92	&	25.70	&	25.85	\\
	60	&	20.35	&	20.60	&	23.07	&	22.32	&	23.11	&	22.40	&	23.36	&	23.37	&	21.33	&	21.69	&	23.95	&	23.26	&	23.91	&	23.25	&	24.21	&	24.16	\\
	80	&	19.64	&	19.81	&	22.02	&	21.40	&	21.76	&	20.76	&	22.32	&	22.30	&	20.68	&	20.93	&	22.90	&	22.42	&	22.60	&	21.78	&	23.19	&	23.12	\\
	100	&	19.24	&	19.35	&	21.23	&	19.80	&	20.97	&	18.15	&	21.56	&	21.52	&	20.29	&	20.49	&	22.07	&	21.85	&	21.75	&	20.83	&	22.37	&	22.39	\\
\hline									
$\sigma$ & \multicolumn{8}{c|}{\emph{House} } & \multicolumn{8}{c|}{\emph{Lena} } \\
\hline
	20	&	28.20	&	30.02	&	32.92	&	31.82	&	32.54	&	31.57	&	32.73	&	32.47	&	26.90	&	28.03	&	29.83	&	29.19	&	30.13	&	28.91	&	29.93	&	30.06	\\
	40	&	23.26	&	24.27	&	28.31	&	27.31	&	28.49	&	26.75	&	28.91	&	28.77	&	22.40	&	23.11	&	26.40	&	25.96	&	26.40	&	25.46	&	26.23	&	26.70	\\
	60	&	21.40	&	21.79	&	26.05	&	24.72	&	26.10	&	24.11	&	26.68	&	26.58	&	20.22	&	20.60	&	24.49	&	23.51	&	24.57	&	22.73	&	24.49	&	24.80	\\
	80	&	20.52	&	20.70	&	24.46	&	22.96	&	23.88	&	22.41	&	25.20	&	25.04	&	19.09	&	19.32	&	23.10	&	22.00	&	22.61	&	21.54	&	23.22	&	23.45	\\
	100	&	20.04	&	20.13	&	23.21	&	20.80	&	22.84	&	21.07	&	23.96	&	23.83	&	18.47	&	18.62	&	22.03	&	20.98	&	21.05	&	20.56	&	22.25	&	22.41	\\
\hline							
$\sigma$ & \multicolumn{8}{c|}{\emph{Man} } & \multicolumn{8}{c|}{\emph{Pepper} } \\
\hline											
	20	&	25.14	&	26.09	&	28.12	&	27.37	&	28.37	&	27.56	&	28.13	&	28.43	&	26.17	&	27.89	&	29.58	&	28.86	&	29.61	&	28.80	&	29.61	&	29.76	\\
	40	&	21.93	&	22.26	&	24.78	&	24.29	&	25.05	&	24.53	&	24.91	&	25.19	&	21.19	&	22.23	&	25.43	&	24.61	&	25.68	&	24.25	&	25.44	&	26.02	\\
	60	&	20.26	&	20.45	&	23.12	&	22.24	&	23.20	&	22.75	&	23.26	&	23.48	&	19.05	&	19.61	&	23.28	&	22.24	&	23.43	&	21.24	&	23.35	&	23.77	\\
	80	&	19.33	&	19.46	&	22.01	&	20.72	&	22.09	&	21.30	&	22.26	&	22.27	&	17.92	&	18.22	&	21.71	&	20.58	&	21.52	&	20.79	&	21.93	&	22.16	\\
	100	&	18.78	&	18.87	&	21.07	&	20.42	&	20.94	&	20.70	&	21.48	&	21.38	&	17.27	&	17.45	&	20.51	&	19.54	&	20.60	&	19.05	&	20.86	&	20.93	\\
\hline	
\end{tabular}
}
\label{table:standard_experiment}
\end{table*}

\begin{table}[t]
\centering
\caption{PSNR comparison between GLIDE and GSF. Results are averaged over 10 testing images.}
\small{
\begin{tabular}{cccc}
\hline
            & \hspace{-0.5cm}  GLIDE                    &  Ours                  &   \\
$\sigma$    & \hspace{-0.4cm} ($\mathrm{PSNR}_1$)       &  ($\mathrm{PSNR}_2$)   &   $\mathrm{PSNR}_2-\mathrm{PSNR}_1$\\
\hline
 20 	 & 	 28.15 	 & 	 28.89 	 & 	 +0.74\\
 40 	 & 	 24.73 	 & 	 25.38 	 & 	 +0.65\\
 60 	 & 	 22.75 	 & 	 23.60 	 & 	 +0.85\\
 80 	 & 	 21.51 	 & 	 22.39 	 & 	 +0.88\\
100 	 & 	 20.42 	 & 	 21.47 	 & 	 +1.05\\
\hline
\end{tabular}
}
\vspace{-4ex}
\label{table:improvement glide}
\end{table}

The full Sinkhorn-Knopp algorithm is implemented using GLIDE \cite{Talebi_Milanfar_2014}. GLIDE consists of multiple steps: It first determines the weight matrix, followed by a full Sinkhorn-Knopp algorithm that symmetrizes the weight matrix. Then, it incorporates an estimator to optimally determine the number of non-zero eigenvalues and the power of eigenvalues of the smoothing filter. GLIDE can use any denoising result as its pilot estimate. For the fairness of the experiment we follow the default setting of GLIDE and use the standard NLM as the pilot estimate. The result in Table~\ref{table:improvement glide} shows that in general GSF has at least 0.65dB improvement over GLIDE. This result is consistent with our observation that full Sinkhorn-Knopp is an incomplete EM algorithm.

\subsection{Comparison with NL-Bayes, PLE and EPLL}
Since the proposed GSF uses a Gaussian mixture model, we compare it with three other algorithms that also use Gaussian mixture models. These algorithms are the NL-Bayes \cite{Lebrun_Buades_Morel_2013}, the piecewise linear estimator (PLE) \cite{Yu_Sapiro_Mallat_2012} and the EPLL \cite{Zoran_Weiss_2011}.

Methodologically, there are some important differences between GSF, NL-Bayes, PLE and EPLL. NL-Bayes has a grouping procedure that groups similar patches, a process similar to BM3D. The grouped patches are used to estimate the empirical mean and covariance of a single Gaussian. In GSF, there is no grouping. The other difference is that the denoising of NL-Bayes is performed by a conditional expectation over the \emph{single} Gaussian. In GSF, the denoising is performed over all clusters of the GMM. Experimentally, we observe that GSF and NL-Bayes have similar performance, with NL-Bayes better in the low noise conditions and GSF better in the high noise conditions. One possible reason is that the grouping of NL-Bayes uses the standard Euclidean distance as a metric, which is not robust to noise.

PLE first estimates the model parameters using the noisy image. Then, for every patch, the algorithm selects a single Gaussian. The denoising is performed by solving a quadratic minimization and is performed for each patch individually. The algorithm iterates by improving the model parameters and the denoised estimate until convergence. GSF does not have this iterative procedure. Once the GMM is learned, the denoising is performed in closed-form. The other difference is that PLE requires a good initialization and is very sensitive to the initialization. Experimentally, we find that GSF performs better than PLE using a MATLAB code provided the authors of PLE. In this MATLAB code, the initialization was originally designed for image inpainting at a particular image resolution. Because of the sensitivity of PLE to the initializations, its performance on denoising is not particularly good. With a better initialization, we believe that PLE would improve. However, even so the gap between GSF and PLE will unlikely be significant because PLE performs worse that BM3D and EPLL.

The closest comparison to GSF is EPLL as both algorithms solve a whole-image MAP minimization with a Gaussian mixture model. To evaluate the difference between the two algorithms, we consider an experiment by feeding the noisy patches the EM algorithm to learn a GMM. The patch size is fixed at $5 \times 5$, and the number of clusters is fixed as $k = 100$. We repeat the experiment by inputting the denoised result of BM3D and the oracle clean image into the EM algorithm.

From Table~\ref{table:improvement epll}, we observe that EPLL with a noisy input performs poorly. The reason is that the original EPLL trains the GMM from 2 million clean patches. When feeded with noisy images, the GMM trained becomes a non-informative prior distribution. Moreover, in EPLL the GMM involves $(\pi_i,\vmu_i,\mSigma_i)$ whereas in GSF the GMM only involves $(\pi_i,\vmu_i)$. This is a significant reduction in the number of model parameters. When feeded with only a single image, there is insufficient training sample for EPLL.

Another observation from Table~\ref{table:improvement epll} is that the performance of EPLL depends heavily on the quality of the GMM. For example, if we use the result of BM3D as a pilot estimate for learning the GMM, the performance of EPLL is similar to the oracle case where we use the clean image. However, using BM3D as a pilot estimate is not a plausible approach because by running BM3D alone we can get an even higher PSNR (See Table~\ref{table:standard_experiment}). This result further shows the effectiveness of the proposed GSF for single image denoising.

\begin{table}[h]
\centering
\caption{Comparison with EPLL using different pilot estimates: ``Noisy'' uses the noisy image; ``BM3D'' uses the BM3D estimate; ``Clean'' uses the oracle clean image; ``External'' uses an external database. Testing image is ``House''.}
\small{
\begin{tabular}{cccccc}
\hline
         &   EPLL    &   EPLL    &   EPLL    &   EPLL       &   Ours\\
$\sigma$ &  (Noisy)  &   (BM3D)  &   (Clean) &   (External) &   \\
\hline
 20 	 & 	 25.40 	 & 	 32.41 	 & 	 32.46 	 & 	 32.47 	 & 	 32.92\\
 40 	 & 	 19.75 	 & 	 28.32 	 & 	 28.31 	 & 	 28.77 	 & 	 28.31\\
 60 	 & 	 16.42 	 & 	 25.73 	 & 	 25.80 	 & 	 26.58 	 & 	 26.05\\
 80 	 & 	 14.29 	 & 	 24.05 	 & 	 24.07 	 & 	 25.04 	 & 	 24.46\\
100 	 & 	 12.71 	 & 	 22.59 	 & 	 22.73 	 & 	 23.83 	 & 	 23.21\\
\hline
\end{tabular}
\vspace{-2ex}
}
\label{table:improvement epll}
\end{table}

\subsection{Complexity and Limitations}
Finally, we discuss the complexity and limitations of the proposed GSF.

GSF is a one-step denoising algorithm when $\gamma_{ij}$ and $\vmu_i^{(r)}$ are known. However, learning the GMM using the EM algorithm is time-consuming, and the complexity depends on the number of clusters $k$. In addition, since $k$ needs to be estimated through a cross-validation scheme, the actual complexity also depends on the number of cross-validation steps. To provide readers an idea of how $k$ changes with other system parameters, we conduct two experiments.

In Table~\ref{table:numcluster sigma} we show the number of clusters returned by the cross-validation scheme as we increase the noise level. As shown, the number of clusters increases when noise level reduces. This result is consistent with our intuition: As noise reduces, the grouping of patches becomes less important. In the limit when the image is noise-free, every patch will become its own cluster center. Therefore, one limitation of GSF is that for low-noise images the computing time could be very long. However, GSF is still a useful tool as its simple structure offers new insights to denoising.

Now, if we fix the noise level but change the image size, the complexity of GSF also varies. In Table~\ref{table:numcluster size}, we show the number of clusters as a function of image size. As a reference we also show the PSNR values of GSF and that of BM3D. The result in Table~\ref{table:numcluster size} indicates that the number of clusters increases with the image size. In the table, we also observe that BM3D performs worse than GSF for small images, but becomes better as image size increases.

\begin{table}[t]
\centering
\caption{Number of clusters returned by cross-validation as noise level increases. Test image is ``Man''. Size is $128 \times 128$.}
\small{
\begin{tabular}{c|ccccccccc}
\hline
$\sigma$ &
20&
30&
40&
50&
60&
70&
80&
90&
100\\
$k$&
1445&
667&
372&
243&
162&
125&
104&
83&
72\\
\hline
\end{tabular}
}
\label{table:numcluster sigma}
\end{table}

\begin{table}[t]
\centering
\caption{Number of clusters returned by cross-validation as image size increases. $\sigma = 40/255$. Test image is ``Man''.}
\small{
\begin{tabular}{cccc}
\hline
     &      & Ours  & BM3D\\
Size &  $k$ & PSNR  & PSNR\\
\hline
$50 \times 50$	    &	120	&	22.76	&	22.36	\\
$100 \times 100$	&	290	&	24.42	&	24.21	\\
$150 \times 150$	&	501	&	25.21	&	25.32	\\
$200 \times 200$	&	778	&	25.82	&	25.99	\\
$250 \times 250$	&	996	&	26.14	&	26.35	\\
$300 \times 300$	&	1322	&	26.58	&	26.83	\\
$350 \times 350$	&	1646	&	26.97	&	27.20	\\
$400 \times 400$    &   1966    &   27.26   &   27.49   \\
\hline
\end{tabular}
}
\label{table:numcluster size}
\end{table}

\begin{remark}[Subtraction of mean]
It is perhaps interesting to ask whether it is possible to subtract the mean before learning the GMM, as it could reduce the number of clusters. However, from our experience, we find that this actually degrades the denoising performance. If the GMM is learned from a collection of zero-mean patches, the denoising step in \eref{eq:gsf minimization} can only be used to denoise zero-mean patches. The mean values, which are also noisy, are never denoised. This phenomenon does not appear in EPLL (in which the GMM has a zero-mean) because the means are iteratively updated. We followed the same approach to iteratively update the means. However, we find that in general the denoising performance is still worse than the original GMM with means included. Further exploration on this would likely provide more insights into the complexity reduction issue.
\end{remark} 
\section{Conclusion}
\label{sec:conclusion}
Motivated by the performance gain due to a column normalization step in defining the smoothing filters, we study the origin of the symmetrization process. Previous studies have shown that the symmetrization process is related to the Sinkhorn-Knopp balancing algorithm. In this paper, we further showed that the symmetrization is equivalent to an EM algorithm of learning a Gaussian mixture model (GMM). This observation allows us to generalize various symmetric smoothing filters including the Non-Local Means (NLM), the one-step Sinkhorn-Knopp and the full Sinkhorn-Knopp, and allows us to geometrically interpret the performance gain phenomenon.

Based on our findings, we proposed a new denoising algorithm called the Gaussian mixture model symmetric smoothing filters (GSF). GSF is a simple modification of the denoising framework by using the GMM prior for the maximum-a-posteriori estimation. Equipped with a cross-validation scheme which can automatically determine the number of clusters, GSF shows consistently better denoising results than NLM, One-step Sinkhorn-Knopp and full Sinkhorn-Knopp. While GSF has slightly worse performance than state-of-the-art methods such as BM3D, its simple structure highlights the importance of clustering in image denoising, which seems to be a plausible direction for future research.

\section*{Acknowledgement}
We thank Dr. Guosheng Yu and Prof. Guillermo Sapiro for sharing the code of PLE \cite{Yu_Sapiro_Mallat_2012}. We also thank the anonymous reviewers for the constructive feedback that helps to significantly improve the paper. 

\appendix
\section{Proofs}

\subsection{Proof of Proposition 1}
Given an estimator $\vzhat$ of some observation $\vy$, the SURE is defined as
\begin{equation}
\mathrm{SURE} \bydef -\sigma^2 + \frac{1}{n}\|\vzhat - \vy\|^2 + \frac{2\sigma^2}{n}\mathrm{div}(\vzhat).
\label{eq:SURE}
\end{equation}
Substituting \eref{eq:zhat lambda} into \eref{eq:SURE} yields
\begin{align}
\frac{1}{n}\|\vzhat - \vy\|^2
&= \frac{1}{n}\left\| \frac{d}{d +\lambda}\vu + \frac{\lambda}{d +\lambda}\vy - \vy\right\|^2 \notag \\
&= \frac{1}{n}\left\| \frac{d}{d +\lambda}(\vu - \vy)\right\|^2 \notag \\
&= \sigmahat^2\left( \frac{d}{d+\lambda}\right)^2, \label{eq:sure 1}
\end{align}
where $\sigmahat^2 \bydef \frac{1}{n}\|\vu - \vy\|^2$. So it remains to determine $\mathrm{div}(\vzhat)$.

From \eref{eq:zhat lambda}, the divergence $\mathrm{div}(\vzhat)$ is
\begin{align*}
\mathrm{div}(\vzhat)
&= \frac{d}{d+\lambda}\mathrm{div}(\vu) + \frac{\lambda}{d+\lambda}\mathrm{div}(\vy)\\
&\bydef \frac{d}{d+\lambda} \sum_{j=1}^n \frac{\partial u_j}{\partial y_j} + \frac{\lambda}{d+\lambda} \sum_{j=1}^n \frac{\partial y_j}{\partial y_j}.
\end{align*}
To determine $\frac{\partial u_j}{\partial y_j}$, we note from \eref{eq:vu}, \eref{eq:vwj} and \eref{eq:em,mu} that
\begin{equation}
\vu = \frac{1}{d}\sum_{j=1}^n\mP_j^T \left( \sum_{i=1}^k\gamma_{ij} \left(\frac{\sum_{j=1}^n \gamma_{ij} \vy_j}{\sum_{j=1}^n \gamma_{ij}}\right) \right).
\end{equation}
Since
\begin{equation*}
\frac{\partial}{\partial y_j}\vy_j
= \frac{\partial}{\partial y_j}
\begin{bmatrix}
\vdots\\
y_{j-1}\\
y_{j}\\
y_{j+1}\\
\vdots
\end{bmatrix}
=
\begin{bmatrix}
\vdots\\
0\\
1\\
0\\
\vdots
\end{bmatrix} = \ve_j,
\end{equation*}
it holds that
\begin{equation*}
\mathrm{div}(\vu)
=  \vone_{n \times 1}^T  \left( \frac{1}{d}\sum_{j=1}^n\mP_j^T \left(\sum_{i=1}^k\gamma_{ij} \left(\frac{\sum_{j=1}^n \gamma_{ij} \ve_j}{\sum_{j=1}^n \gamma_{ij}}\right)\right)\right).
\end{equation*}
and hence
\begin{equation}
\mathrm{div}(\vzhat)
= \sum_{j=1}^n \left( \frac{d}{d+\lambda} \mathrm{div}(\vu) + \frac{\lambda n}{d+\lambda}\right).
\label{eq:div zhat}
\end{equation}
Substituting \eref{eq:div zhat} and \eref{eq:sure 1} into \eref{eq:SURE} completes the proof.

\balance
\bibliographystyle{IEEEbib}
\bibliography{ref_gmm}
\end{document}